\documentclass[letterpaper]{article} 
\usepackage[]{aaai2026_extended}  
\usepackage{times}  
\usepackage{helvet}  
\usepackage{courier}  
\usepackage[hyphens]{url}  
\usepackage{graphicx} 
\usepackage{natbib}  
\usepackage{caption} 
\usepackage{hyperref} 

\usepackage{latexsym}
\usepackage{amssymb}
\usepackage{amsmath}
\usepackage{amsthm}
\usepackage{booktabs}
\usepackage{enumitem}
\usepackage{graphicx}
\usepackage{color}
\usepackage{dsfont}

\usepackage{array}

\usepackage{makecell} 
\usepackage{multirow} 

\usepackage{subcaption}
\newsavebox{\tempbox} 

\usepackage[appendix=strip,bibliography=common]{apxproof}%

\usepackage{thmtools, thm-restate}

\newtheorem{theorem}{Theorem}
\newtheorem{lemma}[theorem]{Lemma}
\newtheorem{corollary}[theorem]{Corollary}
\newtheorem{proposition}[theorem]{Proposition}

\newtheorem{definition}{Definition}
\newtheorem{example}{Example}

\usepackage{tikz}
\usetikzlibrary{arrows,arrows.meta,positioning,plotmarks,bending,quotes,shapes.geometric}
\tikzstyle{link} = []
\tikzstyle{arrow} = [-{Stealth[length=5pt,round,inset=3pt,width=7pt,flex'=1]},line width=0.5pt, tight/.style={inner sep=3pt}, tight2/.style={inner sep=1pt}, tight3/.style={inner sep=2pt}, quotes/.style={fill=white,inner sep=2pt}]
\tikzstyle{node} = [circle,draw=black, minimum size=7mm, inner sep=0pt]

\usepackage[linesnumbered,ruled,vlined]{algorithm2e}
\DontPrintSemicolon

\SetCommentSty{mycommfont}

\usepackage[]{todonotes}

\newcommand{\com}[1]{}

\newcommand{\candidates}{\mathcal{C}}
\newcommand{\cstar}{\candidates\setminus{\{w\}}}
\newcommand{\n}{n}
\newcommand{\m}{m}

\newcommand{\NW}[2]{\text{NW}_{#1}(#2)}
\newcommand{\smo}[4]{\delta_{#4}(#1,#2,#3)}
\newcommand{\sm}[2]{\delta(#1,#2)}

\newcommand{\X}{\mathcal{X}}
\newcommand{\Y}{\mathcal{Y}}

\DeclareMathOperator{\borda}{BO}
\DeclareMathOperator{\cop}{CO}
\DeclareMathOperator{\tc}{TC}
\DeclareMathOperator{\uc}{UC}
\DeclareMathOperator{\wuc}{wUC}
\DeclareMathOperator{\mm}{MM}

\DeclareMathOperator{\WC}{WC}

\newcommand{\edge}[3]{{#1}^{#2\rightarrow #3}}
\newcommand{\overbar}[1]{\mkern 1.5mu\overline{\mkern-1.5mu\!#1\mkern-1.5mu}\mkern 1.5mu}
\newcommand{\wbar}{\overbar{w}}

\title{Explaining Tournament Solutions with Minimal Supports}
\author{
    Clément Contet\textsuperscript{\rm 1}\textsuperscript{\rm 2},
    Umberto Grandi\textsuperscript{\rm 1}\textsuperscript{\rm 3},
    Jérôme Mengin\textsuperscript{\rm 1}\textsuperscript{\rm 2}
}
\affiliations{
    \textsuperscript{\rm 1}Institut de Recherche en Informatique de Toulouse (IRIT)\\
    \textsuperscript{\rm 2}Univeristé de Toulouse\\
    \textsuperscript{\rm 3}Université Toulouse Capitole\\
    \{clement.contet, umberto.grandi, jerome.mengin\}@irit.fr
}

\begin{document}

\maketitle

\begin{abstract}

    Tournaments are widely used models to represent pairwise dominance between candidates, alternatives, or teams. We study the problem of providing certified explanations for why a candidate appears among the winners under various tournament rules. 
    To this end, we identify minimal supports—minimal sub-tournaments in which the candidate is guaranteed to win regardless of how the rest of the tournament is completed (that is, the candidate is a necessary winner of the sub-tournament). This notion corresponds to an abductive explanation for the question,``Why does the winner win the tournament?''—a central concept in formal explainable AI. 
    We focus on common tournament solutions: the top cycle, the uncovered set, the Copeland rule, the Borda rule, the maximin rule, and the weighted uncovered set. For each rule we determine the size of the smallest minimal supports, and we present polynomial-time algorithms to compute them for all solutions except for the weighted uncovered set, for which the problem is NP-complete. Finally, we show how minimal supports can serve to produce compact, certified, and intuitive explanations for tournament solutions.
\end{abstract}


\section{Introduction}\label{sec:intro}

Tournaments are well-known mathematical structures to represent pairwise contests between alternatives, be they a set of candidates in an election, different teams in a sport league, or competing political proposals.
A number of tournament solutions have been proposed and studied in the literature to aggregate the information represented in such graphs and select a set of winning alternatives. For an introduction, see \citet{tournamentHandbook2016} and \citet{weight-tournamentHandbook2016}.

As simple tools for collective decision making, tournament solutions are valuable for digital democracy. However, their adoption will remain limited to small-scale, low-stakes elections until further efforts are made to make them more transparent and accessible to users ~\citep{GrossiDigitalDem2024}.
Leveraging techniques from explainable AI (XAI), in this paper we propose to supplement tournament solutions with a compact certified explanation, enabling voters to efficiently verify the winning alternative while at the same time gaining a clearer understanding of the tournament solution.
XAI aims at helping human users to comprehend and trust the outputs of AI systems. To this end, 
multiple properties have been introduced to define effective explanations
(see, e.g., \citeauthor{miller2019explanation} \citeyear{miller2019explanation}, for an introduction).
First, an explanation should contain the right quantity of information, 
ensuring sufficiency without excess
(see, e.g., the irreducibility and validity properties by~\citeauthor{amgoud2022axiomatic} \citeyear{amgoud2022axiomatic}, and \citeauthor{amgoud2024axiomatic} \citeyear{amgoud2024axiomatic}).
Second, to foster trust, only true statements supported by evidences should be used~\cite{kulesza2013too}. 
Third, the information provided must be relevant to the problem at hand.
Finally, 
an explanation should be brief, organized and non ambiguous (see, e.g., empirical evidence by 
~\citeauthor{delaunay2025impact} \citeyear{delaunay2025impact}). 
\citet{grice1975logic} refers to these four criteria as \emph{quantity}, \emph{quality}, \emph{relation} and \emph{manner}.

Since tournament solutions are white-box models, a natural candidate for an explanation is to provide voters with the full information on the vote counts, to be able to recompute the solution and verify it. 
However, this approach requires more communication effort than necessary, does not provide any new insight to understand the outcome, and is only accessible to experts or at least computer-savvy people for the more advanced tournament solutions.

In the literature, three main options have been proposed to explain and justify the choice returned by a given tournament solution.
First, in social choice and welfare economics (see, e.g., \citeauthor{arrow2010handbook} \citeyear{arrow2010handbook}) one is presented with a set of desirable axiomatic properties or an intuitive mathematical structure that constitutes the normative justification of a tournament solution~\cite{CaillouxEndrissAAMAS2016,boixel2022calculus,NardiEtAlAAMAS2022,peters2020explainable}. 
Yet, explaining why a decision is good does not necessarily help in explaining the actual decision-making process, since multiple outcomes can be supported by different good reasons.
Alternatively, voters can be presented with detailed or aggregated data on their expressed preferences, in the form of statements or features of the election, manually computed or automatically extracted
\cite{SuryanarayanaEtAlAAMAS2022}. Although more accessible, the explanations produced in this way offer no logical guarantee and may be insufficient
to justify a specific result.
Finally, a recent approach by \citet{contet2024abductive} experiments with the use of abductive and contrastive explanations from formal \mbox{explanations} in machine learning (see, e.g., \citeauthor{ijcai2018p708} \citeyear{ijcai2018p708}, \citeauthor{darwiche2020reasons} \citeyear{darwiche2020reasons}, \citeauthor{ignatiev_relating_2020} \citeyear{ignatiev_relating_2020}, \citeauthor{marques-silva_logic-based_2022} \citeyear{marques-silva_logic-based_2022}), introducing them to a matrix-based voting setting. However, these formal explanations often are not user-friendly in view of their size and lack of visual structure.

In this paper we build on the latter proposal to introduce the concept of \emph{minimal supports (MS)} for tournaments solutions, which encapsulate the core reasons why a tournament winner belongs to the winner set. 
We define minimal supports of a winning alternative $w$ for a given tournament solution $S$ on a tournament $G$ as the set of all minimal sub-tournaments of $G$ for which $w$ is a necessary winner, i.e., $w$ is the winner of $S$ in all completions of the MS.
MSs provide a rigorous base to explain the outcome of a tournament, 
as they correspond to \emph{abductive} (or \emph{prime implicant}) \emph{explanations} to the question ``Why is $w$ the $S$-winner of tournament $G$?".
\footnote{Abduction has been defined in a variety of ways in XAI for ML and in logic-based explanations. In particular, sometimes it is limited to “why not” questions (see, e.g., \citeauthor{eiter1995complexity} \citeyear{eiter1995complexity}). In this paper, we follow the terminology of \citet{ignatiev_relating_2020}.}
To show how our explanations can be made accessible, 
we provide efficient algorithms to select the \emph{smallest MSs (SMS)} and we show how to build rigorous, compact and accessible explanations in natural language using MSs.


\setlength{\tabcolsep}{1mm}
\renewcommand{\arraystretch}{1.5}
\begin{table}
    \small
    \centering
    \begin{tabular}{cccc}
        \toprule
        \makecell[c]{Tournament\\ Solution} & \makecell[c]{Computing\\ an SMS} & \makecell[c]{Lower bound\\ on SMS size} & \makecell[c]{Upper bound\\ on SMS size}\\
        \midrule
        Top Cycle & $\mathcal{O}(\m^2)$ & $\m-1$ & $\m-1$ \\
        Uncovered Set & $\mathcal{O}(\m^2)$ & $\m-1$ & $\m-1$ \\
        Copeland & $\mathcal{O}(\m^2)$ & $\m-1$ &
        $(\m - 1)\left\lfloor \frac{\m-1}{2} \right\rfloor$\\
        Borda & $\mathcal{O}(\m^2\log \n)$ & 
        $\left\lceil \n\frac{(\m-1)^2}{\m} \right\rceil$ &
        $(\m - 1)\left\lfloor \n\frac{\m-1}{2} \right\rfloor$\\
        Maximin & $\mathcal{O}(\m^2\log \n)$ & 
        $\left\lceil\frac{\n}{2}\right\rceil(\m-1)$ &
        $\n(\m-1)$\\
        \makecell[c]{Weighted\\ Uncovered Set} & NP-complete & $\n+\m-2$ &
        $(\n+1)(\m-1)$\\
        \bottomrule
    \end{tabular}
    \caption{Overview of our results for various tournament solutions with $\n$ voters and $\m$ candidates. All bounds are tight. Closed-form expressions of the SMSs size are also provided for every solution but the weighted uncovered set.}
    \label{tab:result}
\end{table}

After introducing preliminary notions in 
Section~\ref{sec:prelim}, we introduce in Section~\ref{sec:axpnw} the concept of \emph{minimal support} for tournaments solutions.
While the problem of finding smallest abductive explanations is often intractable (without constraints on the classifier, the decision problem is ${\Sigma_2^\text{P}}$-hard \cite{liberatore2005redundancy}), in Section~\ref{sec:sms} we provide polynomial algorithms for computing \emph{smallest minimal supports} and closed-form expressions of their size for the top cycle ($\tc$), the uncovered set ($\uc$), the Copeland rule ($\cop$), the Borda rule ($\borda$), and the maximin rule ($\mm$). 
We also show that the problem is NP-complete for the weighted uncovered set ($\wuc$).
Additionally, we bound the size of the smallest MSs and show that have underlying structure. In Table~\ref{tab:result} we provide a detailed overview of our technical results. 
Finally, in Section~\ref{sec:explanation} we show how MSs can be used to deliver certified, compact, and intuitive explanations in natural language.
%


\subsection{Related Work}

We identified three approaches in the literature aiming to explain voting outcomes.
%
First, in the realm of voting with rankings, \citet{SuryanarayanaEtAlAAMAS2022} compared experimentally crowdsourced arguments with algorithmic explanations using features of a preference profile such as various notions of score. They showed that both approaches produce comparable results in terms of acceptance and legitimacy of the process, and identified a positive effect on the acceptance of the winner by the less satisfied users.
Second, a stream of papers  mentioned earlier~\cite{CaillouxEndrissAAMAS2016,boixel2022calculus,NardiEtAlAAMAS2022,peters2020explainable} developed a logical calculus based on axiomatic properties to justify why a candidate should be the winner in a preference profile, building on the normative justification of voting rules. 
%
Third, building on recent work in formal XAI, \citet{contet2024abductive} explored the use of abductive explanations in voting with rankings for the Borda rule.

In formal explainability, abductive explanations have a dual structure called contrastive explanations~\cite{ignatiev_relating_2020}. If abductive explanations answer to questions such as ``Why is $a$ the result of classifier $F$ on input $I$?'', contrastive explanations answer questions of the form ``Why is $b$ \emph{not} the result of classifier $F$ on input $I$?''.
These counterfactual questions are answered by identifying minimal sets of input features that have the potential of reverting the classifier result, or that, if left unchanged, can guarantee that the same outcome is reproduced. 
When applied to voting, contrastive explanations are strongly related to the problem of \emph{bribery}, 
studied under various settings ranging from entire changes of a voter's ballot to limited modifications 
(see, e.g., \citeauthor{faliszewski2009hard} \citeyear{faliszewski2009hard}, and \citeauthor{elkind2009swap} \citeyear{elkind2009swap}). The setting considered in this paper, where each pairwise comparison can be individually changed ignoring the transitivity of the preferences, was introduced by \citet{faliszewski2009llull}.

The study of bribery in tournaments has lead \citet{brill2020refining} to define the notion of \emph{margin of victory}, later extended by \citet{doring2023margin}, in an effort 
to refine the outcome of tournament solutions.
Since contrastive explanations are answers to ``\textit{why not?}'' questions, they are great tools to show why a losing alternative was not selected. However, explaining individually why each losing alternative did not win is not necessarily a good explanation of why the winner was elected \cite{lipton1990contrastive}.



%
While bribery is well-studied, its dual problem—defending against it, akin to finding abductive explanations—could have been explored in similar depth. However, to our knowledge, only \citet{chen2021computational} studies the computational complexity of protection against both constructive and destructive bribery, showing that protection is generally harder than the bribery problem itself.

Our approach is in line with the work of \citet{beynier2024explaining}, who provide explanations for the lack of local envy-free allocations in fair division using formal methods akin to abductive explanations and present them with graphs.



\section{Preliminaries}\label{sec:prelim}



Tournaments are well-studied  structures used to represent compactly pairwise comparisons among alternatives. 
In the following presentation we use the framework of \citet{aziz2015possible}.
Formally, a \emph{partial tournament} is a pair $G = (\candidates, E)$ where $\candidates$ is a nonempty finite set of candidates or alternatives and $E \subseteq \candidates \times \candidates$ is an asymmetric relation on $\candidates$, i.e., $(y, x) \not \in E$ whenever $(x, y) \in E$. A \emph{tournament} $G$ is a partial tournament $(\candidates, E)$ for which $E$ is also complete, i.e., either $(x, y) \in E$ or $(y, x) \in E$ for all distinct $x, y \in \candidates$.


Let the \emph{number of voters} $n$ be a strictly positive integer. A \emph{partial $n$-weighted tournament} is a pair $G = (\candidates, \mu)$ where $\candidates$ is a nonempty finite set of candidates and  $\mu : \candidates \times \candidates \to \{0, \dots, n\}$ a weight function such that for all distinct $x, y \in \candidates$, $\mu(x, y) + \mu(y, x) \leq n$. We also assume $\mu(x,x)=0$ for every candidate $x$.
A \emph{(complete) $n$-weighted tournament} satisfies for all distinct $x, y \in \candidates$, $\mu(x, y) + \mu(y, x) = n$.
Naturally, (unweighted) tournaments can be seen as 1-weighted tournaments.

\begin{example}\label{ex:first}
Consider tournament $G=(\candidates,E)$ in Figure~\ref{fig:exunw}, the arrow from node $a$ to node $b$ indicates that $a$ is preferred to $b$, or, that $a$ beats $b$ in a pairwise comparison. In the 5-weighted tournament $G_w=(\candidates,E_w)$ in Figure~\ref{fig:exw}, the ``3'' on the arrow from $a$ to $b$ means that 3 voters prefer $a$ to $b$, or, that $a$ beats $b$ in 3 pairwise comparisons.
\end{example}

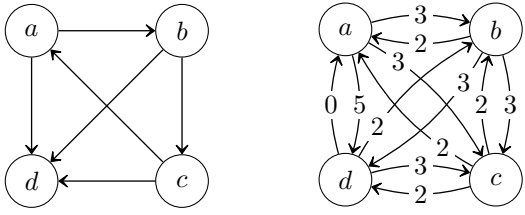
\begin{figure}
    \centering
    \begin{subfigure}[t]{0.475\linewidth}
        \centering
        \begin{tikzpicture}[align=center]
            \node [node] (a) at (0,0) {$a$};
            \node [node] (b) at (2,0) {$b$};
            \node [node] (c)  at (2,-2) {$c$};
            \node [node] (d) at (0,-2) {$d$};
        
            \draw [arrow] (a) to (b) ;
        
            \draw [arrow] (c) to (a) ;
            
            \draw [arrow] (a) to (d) ;
            
            \draw [arrow] (b) to (c) ;
        
            \draw [arrow] (b) to (d) ;
            
            \draw [arrow] (c) to (d) ;
        \end{tikzpicture}
        
        \caption{(unweighted) tournament $G$}\label{fig:exunw}
    \end{subfigure}%
    ~
    \begin{subfigure}[t]{0.475\linewidth}
        \centering
        \begin{tikzpicture}[align=center]
            \node [node] (a) at (0,0) {$a$};
            \node [node] (b) at (2,0) {$b$};
            \node [node] (c)  at (2,-2) {$c$};
            \node [node] (d) at (0,-2) {$d$};
        
            \draw [arrow] (a) to [out=15, in=165] node[quotes]  {$3$} (b) ;
            \draw [arrow] (b) to [out=-165, in=-15] node[quotes]  {$2$} (a) ;
        
            \draw [arrow] (a) [out=-30, in=120] to node[quotes,pos=0.2] {$3$} (c) ;
            \draw [arrow] (c) [out=150, in=-60] to node[quotes,pos=0.2] {$2$} (a) ;
            
            \draw [arrow] (a) [out=-75, in=75] to node[quotes] {$5$} (d) ;
            \draw [arrow] (d) [out=105, in=-105] to node[quotes] {$0$} (a) ;
            
            \draw [arrow] (b) [out=-75, in=75] to node[quotes] {$3$} (c) ;
            \draw [arrow] (c) [out=105, in=-105] to node[quotes] {$2$} (b) ;
        
            \draw [arrow] (b) [out=-120, in=30] to node[quotes,pos=0.2] {$3$} (d) ;
            \draw [arrow] (d) [out=60, in=-150] to node[quotes,pos=0.2] {$2$} (b) ;
            
            \draw [arrow] (d) to [out=15, in=165] node[quotes]  {$3$} (c) ;
            \draw [arrow] (c) to [out=-165, in=-15] node[quotes]  {$2$} (d) ;
        \end{tikzpicture}
        \caption{5-weighted tournament $G_w$}\label{fig:exw}
    \end{subfigure}
    
        \caption{A tournament and a 5-weighted tournament.}
        \label{fig:tournamentsex}
\end{figure}

In the remainder of this paper, edges of weighted tournament with a weight of zero are not represented.



A \emph{tournament solution} (or rule) is a function $S$ mapping a tournament $G=(\candidates, E)$ (or a weighted tournament $G=(\candidates, \mu)$) to a non-empty set of candidates $S(G)\subseteq \candidates$.
In this paper we focus on six common and widely studied rules. 

Let $G =(\candidates,E)$ be a complete tournament. A nonempty subset of candidates $A \subseteq \candidates$ is called \emph{dominant} in $G$ if for every candidate $x\in A$, for all candidates $y\in  \candidates\setminus{A}$, $(x,y) \in E$. A candidate $x\in\candidates$ is said to \emph{cover} another candidate $y\in\candidates$ if for all $z\in\candidates$, $(y,z)\in E \implies (x,z)\in E$.
\begin{itemize}
    \item The \emph{top cycle} ($\tc$) is the (unique) minimal dominant nonempty subset of candidates of $G$, i.e., the top strongly connected component of $G$. Alternatively, it can also be defined as the set of candidates that can reach every other candidate via a directed path in $G$.
    \item The \emph{uncovered set} ($\uc$) is the (unique) nonempty subset of candidates not covered by any other candidate in $G$. Alternatively, it can also be defined as the set of candidates that can reach every other candidate via a directed path in $G$ of length at most two.
    \item The \emph{Copeland score} of a candidate $c\in\candidates$ in $G = (\candidates, E)$ is $\sigma_{\cop}(c,G) = |\{c':c'\in \candidates, (c,c')\in E\}|$. The \emph{Copeland rule} ($\cop$) selects candidates that have a maximal Copeland score.
\end{itemize}

Given $\n$ voters, let $G =(\candidates,\mu)$ be a complete $\n$-weighted tournament. A candidate $x\in\candidates$ is said to \emph{weighted cover} another candidate $y\in\candidates$ if for all $z\in\candidates$, $\mu(x,z)\geq \mu(y,z)$.
\begin{itemize}
    \item The \emph{Borda score} of a candidate $c\in\candidates$ in $G=(\candidates, \mu)$ is $\sigma_{\borda}(c,G) = \sum_{c'\in\candidates}\mu(c,c')$. The \emph{Borda rule} ($\borda$) selects candidates that have a maximal Borda score.
    \item The \emph{maximin score} of a candidate $c\in\candidates$ in $G=(\candidates, \mu)$ is $\sigma_{\mm}(c,G) = \min_{c'\in \candidates\setminus{\{c\}}} \mu(c,c')$. The \emph{maximin rule} ($\mm$) selects candidates that have a maximal maximin score.
    \item The \emph{weighted uncovered set} ($\wuc$) is the nonempty subset of candidates that are not weighted covered by any other candidate in $G$. A path-based alternative definition also exists (see Lemma~\ref{lem:wucchara}).
\end{itemize}

Observe that these six tournament solutions can be organized in two families. First, there are the \emph{path-based solutions}: $\tc$, $\uc$ and $\wuc$ which are solutions associated to some notion of paths. Second, $\cop$, $\borda$ and $\mm$ are rules which score candidates where each candidate's score can be computed using only the pairwise comparisons between it and the other candidates. We call them \emph{myopic-score solutions}. This distinction will become useful in Section~\ref{sec:chosingms} when discriminating between MSs.





\section{Minimal Supports for Tournaments}\label{sec:axpnw}

This section adapts the definitions from the literature on formal explanations to the case of tournaments, and proposes three principles to choose among minimal supports. 

\subsection{MSs for Unweighted Tournaments}

Abductive reasoning applied to tournament solutions explains the winner of a tournament by exhibiting a minimal subset of the tournament such that the winner remains unchanged.
Let $G = (\candidates, E)$ and $G' = (\candidates, E')$ be two partial tournaments, we say that $G'$ is an \emph{extension} of $G$, denoted $G \subseteq G'$, if $E \subseteq E'$. If $E'$ is complete, $G'$ is called a \emph{completion} of $G$, and we write $[G]$ for the set of completions of $G$.
To obtain a formal definition we use the
concept of necessary winner~\cite{konczak_voting_2005}.

\begin{definition}\label{def:NW}
    Given a partial tournament $G$, a tournament solution $S$  and a candidate $c \in \candidates$, $c$ is a \emph{necessary winner} of $S$ in $G$ if for every completion $G'\in [G]$ we have that $c \in S(G')$. We write $c \in \NW{S}{G}$.
\end{definition}

We now define minimal supports for tournaments:

\begin{definition}\label{def:axp}
Given a tournament $G = (\candidates, E)$, a tournament solution $S$ and a winning candidate $w \in S(G)$, a \emph{minimal support (MS)} for $w\in S(G)$ is a partial tournament $G'\subseteq G$ such that: \begin{enumerate}
    \item[(a)]  $w\in \NW{S}{G'}$
    \item[(b)] $G'$ is $\subseteq$-minimal, i.e., all partial tournaments $G''\subsetneq G'$ are such that $w\not\in \NW{S}{G''}$.
\end{enumerate} 
\end{definition}


\begin{example}\label{ex:second}
Consider tournament $G$ in Figure~\ref{fig:exunw}. The partial tournament $\X$ represented in Figure~\ref{fig:X} is an MS for $a\in\uc(G)$.
Note that $\X \subseteq G$.
What $\X$ shows is that 
$a$ is not covered by $b$ since $a$ is preferred to $b$; 
$a$ is not covered by $c$ since $a$ is preferred to $b$ and $b$ is preferred to $c$; 
$a$ is not covered by $d$ since $a$ is preferred to $d$.
Hence, $a$ is in the uncovered set.
Similar reasoning holds for $\Y$ in Figure~\ref{fig:Y}.
\end{example}

\begin{figure}
     \centering
     \begin{subfigure}[t]{0.475\linewidth}
        \centering
        \begin{tikzpicture}[align=center] 
            \node [node] (a) at (0,0) {$a$};
            \node [node] (b) at (2,0) {$b$};
            \node [node] (c)  at (2,-2) {$c$};
            \node [node] (d) at (0,-2) {$d$};
        
            \draw [arrow] (a) to node[above,tight2] {} (b) ;
            \draw [arrow] (a) to node[left, tight2] {} (d) ;
            \draw [arrow] (b) to node[right, tight2] {} (c) ;
        \end{tikzpicture}
        \caption{$\X$}\label{fig:X}
    \end{subfigure}%
    ~
    \begin{subfigure}[t]{0.475\linewidth}
        \centering
        \begin{tikzpicture}[align=center] 
            \node [node] (a) at (0,0) {$a$};
            \node [node] (b) at (2,0) {$b$};
            \node [node] (c)  at (2,-2) {$c$};
            \node [node] (d) at (0,-2) {$d$};
        
            \draw [arrow] (a) to node[above,tight2] {} (b) ;
            \draw [arrow] (b) to node[right, tight2] {} (c) ;
            \draw [arrow] (b) to node[right, tight2] {} (d) ;
        \end{tikzpicture}
        \caption{$\Y$}\label{fig:Y}
    \end{subfigure}

    \caption{MSs for $a\in\uc(G)$ from Figure~\ref{fig:exunw}.}
        \label{fig:axps_unw_ex}
\end{figure}
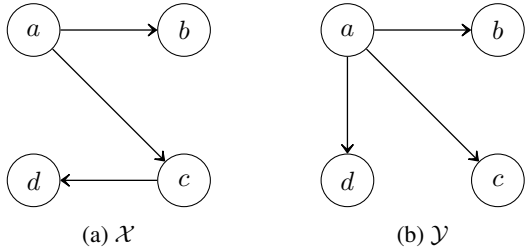

\subsection{MSs for Weighted Tournaments}

The case of weighted tournaments can be treated analogously by simply defining an inclusion relation for weighted tournaments. Let $G = (\candidates, \mu)$ and $G' = (\candidates, \mu')$ be two partial $n$-weighted tournaments, we say that $G'$ is an \emph{extension} of $G$, denoted $G \subseteq G'$, if $\mu(x,y)\leq \mu'(x,y)$ for all $(x,y)\in \candidates\times\candidates$. 
A \emph{completion} of $G$ is a complete $n$-weighted extension, and both Definitions~\ref{def:NW} and~\ref{def:axp} apply.

\begin{example}\label{ex:third}
Consider $G_w$ in Figure~\ref{fig:exw}, both $\mathcal{I}$ and $\mathcal{J}$, represented in Figure~\ref{fig:axps_w_ex}, are MSs for $a\in\mm(G_w)$. Observe that $\mathcal{I} \subseteq G_w$ and $\mathcal{J} \subseteq G_w$.
$\mathcal I$ shows that $a$ achieves a maximin score of 3 and that no other candidate can beat it since they all end up with 3 defeats against $a$. $\mathcal J$ only secures a maximin score of 2 for $a$ and uses additional pairwise comparisons among other candidates to show that they cannot outperform $a$.
\end{example}

\begin{figure}
     \centering
     \begin{subfigure}[t]{0.475\linewidth}
        \centering
        \begin{tikzpicture}[align=center]
            \node [node] (a) at (0,0) {$a$};
            \node [node] (b) at (2,0) {$b$};
            \node [node] (c)  at (2,-2) {$c$};
            \node [node] (d) at (0,-2) {$d$};
        
            \draw [arrow] (a) to node[quotes]  {$3$} (b) ;
        
            \draw [arrow] (a) to node[quotes] {$3$} (c) ;
            
            \draw [arrow] (a) to node[quotes] {$3$} (d) ;
            
        
            
        \end{tikzpicture}
        \caption{$\mathcal{I}$}\label{fig:I}
    \end{subfigure}%
    ~
    \begin{subfigure}[t]{0.475\linewidth}
        \centering
         \begin{tikzpicture}[align=center] 
            \node [node] (a) at (0,0) {$a$};
            \node [node] (b) at (2,0) {$b$};
            \node [node] (c)  at (2,-2) {$c$};
            \node [node] (d) at (0,-2) {$d$};
        
            \draw [arrow] (a) to node[quotes]  {$3$} (b) ;
        
            \draw [arrow] (a) to node[quotes,pos=0.2] {$2$} (c) ;
            
            \draw [arrow] (a) to node[quotes] {$2$} (d) ;
            
            \draw [arrow] (b) to node[quotes] {$3$} (c) ;
        
            \draw [arrow] (b) to node[quotes,pos=0.2] {$3$} (d) ;
            
        \end{tikzpicture}
        \caption{$\mathcal{J}$}\label{fig:J}
    \end{subfigure}

    \caption{MSs for $a\in\mm(G_w)$ from Figure~\ref{fig:exw}.}
    \label{fig:axps_w_ex}
\end{figure}
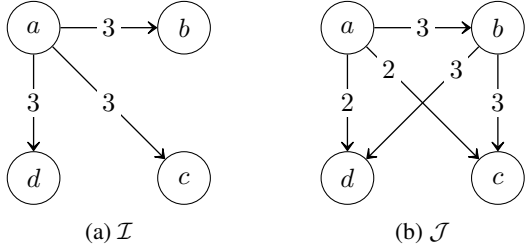

\subsection{Selecting among MSs for Explanations}\label{sec:chosingms}
 

For a given tournament there might be an exponential number of different minimal supports since MSs are an instance of a Sperner family, a family of subsets which do not contain each others \cite{sperner1928satz}. Therefore in this section we introduce criteria to select the most promising MSs from an explainability point of view.

Every MS can be a certificate for an explanation since by definition it is subset minimal and presents precisely enough information to guarantee the outcome of the tournament. Hence, looking back at the Gricean principles~\cite{grice1975logic} mentioned in the introduction, all MSs are optimal for quality. However, we can still use the manner criteria to discriminate between MSs since some can be smaller than others, i.e., rely on a smaller total amount of information. 
This is interesting from the perspective of formal verification since it is easier to verify the correctness of smaller certificates (imagine you have to call back participants or re-watch the games between contestants to verify that the registered preferences are correct).

To measure this amount of information in a general context, we take inspiration from the bribery setting and we fix the smallest atom of information to be the individual swap, i.e., the preference of one voter in the pairwise comparison between two candidates.

\begin{definition}\label{def:sms}
    Given $n$ voters, a partial \mbox{$n${-}weighted} tournament $G =(\candidates,\mu)$, and a winning candidate $w \in S(G)$, we define the size of an MS $\X=(\candidates,\mu_\X)$ for $w \in S(G)$ as $|\X| = \sum_{(c,c') \in \candidates^2} \mu_\X(c,c')$. $\X$ is a \emph{smallest minimal support (SMS)} for $w \in S(G)$ if and only if for all MSs $\Y$ for $w \in S(G)$, we have $|\X| \leq |\Y|$.
\end{definition}

\begin{example}\label{ex:forth}
In Figure~\ref{fig:axps_unw_ex}, $|\X|=3$, $|\Y|=3$ and both are SMSs for $a \in \uc(G)$.
Both $\mathcal{I}$ and $\mathcal{J}$ in Figure~\ref{fig:axps_w_ex} are MSs for $a\in\mm(G_w)$. However, $\mathcal{I}$ is smaller than $\mathcal{J}$ as $|\mathcal{I}|=9$ and $|\mathcal{J}|=13$. Only $\mathcal{I}$ is an SMS for $a\in\mm(G_w)$.
\end{example}

Smallest MSs for weighted tournaments can be viewed as MSs such that no strictly smaller MS exists for the $\ell_1$-norm ($||x||_1 = \sum_{i\in I} |x_i|$) of the tournament weight vector.
For unweighted tournaments this reduces to minimizing the number of edges.

To further select MSs, we propose two additional criteria to choose among SMSs depending of the type of tournament solution. 
For the path-based solutions, it is natural for simplicity purposes to look for SMSs containing the shortest paths between the winning candidate and the other ones. 

\begin{definition}
    Given a complete tournament $G =(\candidates,\mu)$, a tournament solution $S$, and a winning candidate $w \in S(G)$, a \emph{shortest-path} SMS $\X$ for $w \in S(G)$, is an SMS such that for all $c\in\candidates\setminus{\{w\}}$, the directed path from $w$ to $c$ in $\X$ is the shortest in $G$.
\end{definition}

\begin{example}\label{ex:fivth}
Looking back at Figure~\ref{fig:axps_unw_ex}, the path from $a$ to $d$ is 1-long in $\X$ and 2-long in $\Y$ as it goes through $b$. Hence, we will prefer $\X$ with its shorter paths.
\end{example}

For the myopic-score solutions, we argue that it is better to put more emphasis on the good performances of the winner rather than its opponents bad ones. Typically, if a candidate is a Condorcet winner, we would like to show it like $\mathcal{I}$ in Figure~\ref{fig:I}. In this perspective, we will prioritize SMSs containing the highest number of won pairwise comparisons by the considered tournament winner.

\begin{definition}\label{def:maxwin-sms}
    Given $n$ voters, a partial \mbox{$n${-}weighted} tournament $G =(\candidates,\mu)$, a tournament solution $S$, and a winning candidate $w \in S(G)$. Given an SMS $\X=(\candidates,\mu_\X)$ for $w \in S(G)$, we define the win count of $\X$ as $\WC(\X)=\sum_{c\in\candidates}\mu_X(w,c)$. $\X$ is a \emph{maxwin-SMS} if and only if for all SMSs $\Y$ for $w \in S(G)$, we have $\WC(\X) \geq \WC(\Y)$.
\end{definition}



\section{Computing SMS}\label{sec:sms}

In this section we present the technical results of this paper, characterizing the MSs and SMSs for common tournament solutions, and providing polynomial algorithms for their computation when this is possible. All missing proofs can be found in the \hyperref[appendix]{Appendix} of this paper.

\subsection{Top Cycle and Uncovered Set}\label{sec:tcuc}



Top cycle ($\tc$) selects the smallest dominating set of candidates, or, in graph-theoretic perspective, the top strongly-connected component. Here we use an alternative characterization based on paths to prove that all MSs for $\tc$ are directed trees.

\begin{proposition}\label{prop:tcstrcut}
    Given a complete tournament $G =(\candidates,E)$, and a winning candidate $w \in \tc(G)$, for all MSs $\X$ for $w \in \tc(G)$, $\X$ is a \emph{$w$-rooted out-tree}, i.e., for all $c\in\candidates\setminus{\{w\}}$, there exists a \emph{unique} directed path from $w$ to $c$ in $\X$.
\end{proposition}

\begin{proof}
    Let $\X=(C,E_\X)$ be an MS for $w \in \tc(G)$.
    Suppose there exists $c\in\candidates\setminus{\{w\}}$ such that there is no directed path from $w$ to $c$ in $\X$.
    Let $\Y$ be a completion of $\X$ such that $w$ and all the candidates reachable from $w$ in $\X$ lose their missing pairwise comparisons, i.e. such that $\forall (u,v)\in R_w\times\candidates$, $(u,v) \in \Y$ implies $(u,v) \in \X$ or $v\in R_w$ where $R_w$ is the set of reachable candidates from $w$.
    Then no new candidate is reachable from $w$ in $\Y$. Thus, there is no path in $\Y$ from $w$ to $c$. Hence, $w\not\in\tc(\Y)$ and $w\not\in\NW{\tc}{\X}$ and $\X$ is not a an SMS for $w \in \tc(G)$. Contradiction.
    Hence, for all $c\in\candidates\setminus{\{w\}}$, there exists a directed path from $w$ to $c$ in $\X$.
    
    Let us now show that for all $c\in\candidates\setminus{\{w\}}$, there exists a \emph{unique} directed path from $w$ to $c$ in $\X$.
    Suppose that there exists $c^*\in\candidates\setminus{\{w\}}$ such that there are two directed paths $p'$ and $p''$ from $w$ to $c^*$ in $\X$ such that $p'$ ends with $(c',c^*) \in E_\X$, $p''$ ends with $(c'',c^*) \in E_\X$ and $c'\neq c''$.
    For all $c\in\candidates\setminus{\{w\}}$, we know there exists a directed path $p_c$ from $w$ to $c$ in $\X$. 
    Let $\Y=(\candidates,E_\Y)$ be such that $E_\Y=E_\X\setminus{\{(c'',c^*)\}}$. 
    If $p_c$ does not contain $(c'',c^*)$, $\Y$ also contains $p_c$. 
    Else, if $p_c$ contains $(c'',c^*)$, $p_c$ can be decomposed as $p''\circ p_c'$ where $\circ$ is the concatenation and $p_c'$ a path from $c^*$ to $c$. $p'\circ p_c'$ is a directed path from $w$ to $c$ in $\Y$.
    For all $c\in\candidates\setminus{\{w\}}$, there exists a directed path from $w$ to $c$ in $\Y$, thus $w\in\NW{\tc}{\Y}$ but $\Y\subsetneq\X$ so $\X$ is not an MS for $w \in \tc(G)$. Contradiction.
\end{proof}

From the common structure shared by MSs for $\tc$, we can show that all MSs and thus SMSs have the same size.

\begin{restatable}{theorem}{tcsize}\label{th:tcsize}
    Given a complete tournament $G =(\candidates,E)$ with $|\candidates| = \m$, and a winning candidate $w \in \tc(G)$, for all MSs $\X$ for $w \in \tc(G)$, we have $|\X|{=} \m-1$.
\end{restatable}


Finally, computing a shortest-path SMS for $\tc$ can be achieved in polynomial time with Dijkstra's algorithm \cite{dljkstra1959note}.

\begin{restatable}{proposition}{tcalgo}\label{prop:tcalgo}
    Given a complete tournament $G =(\candidates,E)$ with $|\candidates| = \m$, and a winning candidate $w\in\tc(G)$, there exists an algorithm which computes a shortest-path SMS for $w\in\tc(G)$ in $\mathcal{O}(\m^2)$.
\end{restatable}





Like $\tc$, the uncovered set ($\uc$) has a characterization based on the existence of directed paths between the winner and the other candidates, except that these paths have a maximum length of two. Hence, the following results can be derived from the previous propositions.

\begin{corollary}\label{prop:ucstrcut}
    Given a complete tournament $G =(\candidates,E)$, and a winning candidate $w \in \uc(G)$, for all MSs $\X$ for $w \in \uc(G)$, $\X$ is a $w$-rooted out-tree of depth at most 2, i.e., for alls $c\in\candidates\setminus{\{w\}}$, there exists a \emph{unique} directed path of length at most 2 from $w$ to $c$ in $\X$.
\end{corollary}

\begin{corollary}\label{prop:ucsize}
    Given a complete tournament $G =(\candidates,E)$ with $|\candidates| = \m$, and a winning candidate $w \in \uc(G)$, for all MSs $\X$ for $w \in \uc(G)$, we have $|\X|{=} \m-1$.
\end{corollary}

\begin{corollary}\label{prop:ucalgo}
    Given a complete tournament $G =(\candidates,E)$ with $|\candidates| = \m$, and a winning candidate $w\in\uc(G)$, there exists an algorithm which computes a shortest-path SMS for $w\in\uc(G)$ in $\mathcal{O}(\m^2)$.
\end{corollary}







\subsection{Weighted Uncovered Set}\label{sec:wuc}



To construct SMSs for the weighted uncovered set ($\wuc$), one can generalize the alternative characterization of the uncovered set to the case of weighted tournaments given by \citet{doring2023margin}.

\begin{restatable}[Lemma 2.2 in \citeauthor{doring2023margin} \citeyear{doring2023margin}]{lemma}{wucchara}\label{lem:wucchara}
    Given $\n$ voters, a complete $\n$-weighted tournament $G =(\candidates,\mu)$, the \emph{weighted uncovered set} is the nonempty subset of candidates that can reach every other candidate via a directed path $(c,c')$ such that $\mu(c,c')>\mu(c',c)$ or a directed path $(c,c',c'')$ such that $\mu(c,c')>\mu(c'',c')$ in $G$.
\end{restatable}

From Lemma~\ref{lem:wucchara}, we obtain a result similar to Corollary~\ref{prop:ucstrcut} on the structure of any minimal support for $\wuc$.

\begin{restatable}{proposition}{wucstrcut}\label{prop:wucstrcut}
    Given $\n$ voters, a complete $\n$-weighted tournament $G =(\candidates,\mu)$, a winning candidate $w \in \wuc(G)$, for all MSs $\X=(\candidates,\mu_\X)$ for $w \in \wuc(G)$, $\X$ is a $w$-rooted out-tree of depth at most 2 such that for all losing candidates $c\in\candidates\setminus{\{w\}}$, $\mu_\X(w,c)\geq\lceil(\n+1)/2\rceil$ or there exists a unique $c'$ such that $\mu_\X(w,c')+\mu_\X(c',c)\geq\n+1$.
\end{restatable}

Even if MSs for $\wuc$ follow a structure similar to the unweighted case, we now show that deciding whether there exists an MS smaller than a given size is NP-complete.

\begin{restatable}{theorem}{wucsizehard}\label{th:wucsizehard}
   Given $\n\geq2$ voters, a complete $\n$-weighted tournament $G =(\candidates,\mu)$, a winning candidate $w \in \wuc(G)$, and an integer $k$, determining if there exists an MS of size at most $k$ is NP-complete.
\end{restatable}

\begin{proof}[Proof sketch]
    To prove NP membership, note that Proposition~\ref{prop:wucstrcut} gives an easy to verify structure to identify if a partial tournament is an MS for $\wuc$. 
    
    To prove NP-hardness, we proceed with a reduction from the NP-complete \textsc{Set Cover} problem \cite{karp1972reducibility}. Given an instance of set cover problem with a universe $U$ of $p$ elements and a set $S$ of $q$ subsets of $U$, we show that there exists a set cover of size less than $k\in\mathbb{N}$ if and only if there exists an MS for $\wuc$ of size less than $k'=p+q+k$ for a designated winner in a specific tournament with candidates associated to each element of $U$ and $S$.
\end{proof}



Even if computing SMSs for $\wuc$ is a hard problem, we provide tight bounds on their size.

\begin{restatable}{proposition}{wucbounds}\label{cor:wucbounds}
    Given $\n\geq2$ voters, a complete $\n$-weighted tournament $G =(\candidates,\mu)$ with $|\candidates| = \m\geq3$, and a winning candidate $w \in \wuc(G)$, for all SMSs $\X$ for $w \in \wuc(G)$, we have $\n + \m - 2 \leq |\X|\leq(\n+1)(\m-1)$.
\end{restatable}


\subsection{Maximin}\label{sec:mm}


The maximin rule ($\mm$) selects the candidates with the least worst performance in a head-to-head against another alternative. Thus, an MS for $\mm$ must guarantee overall good performances for the winner while ensuring that the other candidates perform poorly at least once.

\begin{restatable}{lemma}{mmchara}\label{lem:mmchara}
     Given $\n$ voters, a complete $\n$-weighted tournament $G =(\candidates,\mu)$, and a winning candidate $w \in \mm(G)$ with a maximin score of $\sigma_w$, for all MSs $\X=(\candidates,\mu_\X)$ for $w \in \mm(G)$, there exists $t\leq\left\lceil\n/2\right\rceil$ such that, for all $c\in\candidates\setminus{\{w\}}$, $\mu_\X(w,c)\geq t$ and there exists $c'\in\candidates$ such that $\mu_\X(c',c)\geq \n-t$.
\end{restatable}

From the previous characterization, we can derive a closed-form expression for the size of an SMS for $\mm$.

\begin{theorem}\label{th:mmsize}
    Given $\n$ voters, a complete $\n$-weighted tournament $G =(\candidates,\mu)$ with $|\candidates| = \m$, and a winning candidate $w \in \mm(G)$ with a maximin score of $\sigma_w$, for all SMSs $\X$ for $w \in \mm(G)$, we have $|\X|{=} \n(\m-1) -t |\{c:c\in\candidates,\,\mu(w,c)\geq \n-t\}|$ where $t= \min(\sigma_w, \left\lfloor\n/2\right\rfloor)$.
\end{theorem}

\begin{proof}
    According to Lemma~\ref{lem:mmchara}, for each candidate different from $w$, $n=t+n-t$ pairwise comparisons are needed.
    However, one can optimize the size of the MS when $w$ wins $\n-t$ pairwise comparisons against an opponent $c$ by covering both constraints at once.
    Indeed, let $\X=(\candidates,\mu_\X)$ be an SMS for $w \in \mm(G)$ and $t \leq \lfloor \n /2 \rfloor$. For all $c\in \candidates$ such that $\mu(w,c)\geq\n-t$, if we have $\mu_\X(w,c)=\n-t$, then for all completions $\X'$ of $\X$, we have $\mu_{\X'}(w,c)\geq \n-t \geq t$.
    For all the other candidates $c\in \candidates$ such that $\mu(w,c)<\n-t$, we take $\mu_\X(w,c)=t$ and $\mu_\X(c',c)=\n-t$ for some candidate $c'\in\candidates$ such that $\mu(c',c)\geq\n-t$. Then for all completions $\X'$ of $\X$, we have $\mu_{\X'}(w,c)\geq t$ and $\mu_{\X'}(c',c)\geq \n-t$.
    Let $t= \min(\sigma_w, \left\lfloor\n/2\right\rfloor)$, for all losing candidates $c\in\candidates\setminus{\{w\}}$, either $\mu(w,c)\geq\n-t$ and only $n-t$ pairwise comparisons are needed else $\n=\n-t+t$ comparisons are needed.
    With $k=|\{c:c\in\candidates,\,\mu(w,c)\geq \n-t\}|$, we have $|\X|=(\n-t)k+\n(\m-1-k)=\n(\m-1)-t k$.
\end{proof}

We provide tight bounds on the size of an SMS for $\mm$.

\begin{restatable}{corollary}{mmbounds}\label{cor:mmbounds}
    Given $\n$ voters, a complete $\n$-weighted tournament $G =(\candidates,\mu)$ with $|\candidates| = \m$, and a winning candidate $w \in \mm(G)$ with a maximin score of $\sigma_w$, for all SMSs $\X$ for $w \in \mm(G)$, we have $\left\lceil\frac{\n}{2}\right\rceil(\m-1) \leq |\X|{\leq} \n(\m-1)$.
\end{restatable}

Computing a maxwin-SMS for $\mm$ can be achieved by analyzing the local neighborhood of each candidate.

\begin{restatable}{proposition}{mmalgo}\label{prop:mmalgo}
    Given $\n$ voters, a complete $\n$-weighted tournament $G =(\candidates,\mu)$ with $|\candidates| = \m$ and a winning candidate $w\in\mm(G)$, there exists an algorithm which computes a maxwin-SMS for $w\in\mm(G)$ in $\mathcal{O}(\m^2\log\n)$.
\end{restatable}

\subsection{Borda and Copeland Rules}\label{sec:borda}


The Borda rule ($\borda$) chooses the set of candidates with a maximal total amount of wins across all pairwise comparisons with the other candidates.


We show that the size of an SMS for $\borda$ only depends on the Borda score of the winner and its worst performance against the other candidates. 

\begin{restatable}{theorem}{bordasize}\label{th:bordasize}
    Given $n$ voters, a complete \mbox{$\n${-}weighted} tournament $G =(\candidates,\mu)$ with $|\candidates| = \m$, and a winning candidate $w \in \borda(G)$ with a Borda score of $\sigma_w$, for all SMSs $\X$ for $w \in \borda(G)$:
    \begin{itemize}
        \item if $\sigma_w \leq \n(\m-1)-\displaystyle\min_{c\neq w}\mu(w,c)$ then\\
        $|\X|{=}(\m-1)(\n(\m-1)-\sigma_w)$
        \item else $|\X|{=} \n(\m-1) - \min(\lfloor \n(1-\frac{1}{\m})\rfloor, \displaystyle\min_{c\neq w}\mu(w,c))$.
    \end{itemize}
\end{restatable}


From the characterization of the size of SMSs for $\borda$, we derive tight upper and lower bounds on SMSs size.

\begin{restatable}{corollary}{bordabounds}\label{cor:bordabounds}
    Given $n$ voters, a partial \mbox{$\n${-}weighted} tournament $G =(\candidates,\mu)$ with $|\candidates| = \m$, and a winning candidate $w \in \borda(G)$. For all SMSs $\X$ for $w \in \borda(G)$, we have 
    $ \left\lceil \frac{\n(\m-1)^2}{\m} \right\rceil \leq |\X| \leq (\m - 1)\left\lfloor \n\frac{\m-1}{2} \right\rfloor$.
\end{restatable}

We now show that we can compute a maxwin-SMS in polynomial time.

\begin{restatable}{proposition}{bordaalgo}\label{prop:bordaalgo}
    Given $\n$ voters, a complete $n$-weighted tournament $G =(\candidates,E)$ with $|\candidates| = \m$, and a winning candidate $w\in\borda(G)$, there exists an algorithm which computes a maxwin-SMS for $w\in\borda(G)$ in $\mathcal{O}(\m^2\log \n)$.
\end{restatable}

Finally, we show that all SMSs share a similar structure centered around the winner's wins by bounding the difference in terms of win count between SMSs.

\begin{restatable}{proposition}{wcstructure}\label{prop:wcstructure}
    Given $n$ voters, an $n$-weighted tournament $G =(\candidates,\mu)$, and a winning candidate $w\in\borda(G)$, for all SMSs $\X,\Y$ for $w\in\borda(G)$, we have $|\WC(\X)-\WC(\Y)|\leq\max(1,\frac{n}{2})$.
\end{restatable}

Since the Copeland rule ($\cop$) coincides with the Borda rule on 1-weighted tournaments, the following results are a special case of the previous results obtained on Borda.

\begin{restatable}{corollary}{copsize}\label{prop:copsize}
    Given a complete tournament $G =(\candidates,E)$ with $|\candidates| = \m$ and a winning candidate $w\in\cop(G)$ with a Copeland score of $\sigma_w$, for all SMSs $\X$ for $w\in\cop(G)$:
    \begin{itemize}
        \item if $w$ is a Condorcet winner then $|\X|{=}\m-1$
        \item else $|\X|{=}(\m-1)(\m-1-\sigma_w)$.
    \end{itemize}
\end{restatable}

\begin{corollary}\label{cor:copbounds}
    Given a complete tournament $G =(\candidates,E)$ with $|\candidates| = \m$ and a winning candidate $w\in\cop(G)$, for all SMSs $\X$ for $w\in\cop(G)$
    $\m-1 \leq |\X| \leq (\m - 1)\left\lfloor\frac{\m-1}{2} \right\rfloor$.
\end{corollary}

\begin{restatable}{corollary}{copalgo}\label{prop:copalgo}
    Given a complete tournament $G =(\candidates,E)$ with $|\candidates| = \m$ and a winning candidate $w\in\cop(G)$, there exists an algorithm which computes a maxwin-SMS for $w\in\cop(G)$ in $\mathcal{O}(\m^2)$.
\end{restatable}

\begin{figure*}
    \centering
     \begin{subfigure}[b]{0.15\linewidth}
        \centering
        \begin{tikzpicture}[align=center]
            \node [node] (a) at (0,0) {$a$};
            \node [node] (b) at (2,0) {$b$};
            \node [node] (c)  at (2,-2) {$c$};
            \node [node] (d) at (0,-2) {$d$};
        
            \draw [arrow] (a) to (b) ;
        
            \draw [arrow] (c) to (a) ;
            
            \draw [arrow] (a) to (d) ;
            
            \draw [arrow] (b) to (c) ;
        
            \draw [arrow] (b) to (d) ;
            
            \draw [arrow] (c) to (d) ;
        \end{tikzpicture}
        \caption{}
    \end{subfigure}%
    ~
    \begin{subfigure}[b]{0.15\linewidth}
        \centering
        \begin{tikzpicture}[align=center]
            \node [node] (a) at (0,0) {$a$};
            \node [node] (b) at (2,0) {$b$};
            \node [node] (c)  at (2,-2) {$c$};
            \node [node] (d) at (0,-2) {$d$};
        
            \draw [arrow] (a) to (b) ;
        
            
            \draw [arrow] (a) to (d) ;
            
            \draw [arrow] (b) to (c) ;
        
            
        \end{tikzpicture}
        \caption{}
    \end{subfigure}
    ~
    \begin{subfigure}[b]{0.13\linewidth}
        \centering
        \begin{tikzpicture}[align=center, sibling distance=15mm, level distance = 10mm]
            \node [node] (a) at (0,0) {$a$}
                [arrow] child {node [node] (b) {$b$} 
                    [arrow] child {node [node] (c) {$c$}}}
                child {node [node] (d) {$d$}};
        \end{tikzpicture}
        \caption{}\label{fig:tree}
    \end{subfigure}
    ~
    \begin{subfigure}[b]{0.53\linewidth}
        \centering
        \begin{minipage}[]{1\linewidth}
            $a$ is part of the uncovered set because
            \begin{itemize}
                \item $a$ is not covered by $b$ since $a$ is preferred to $b$ 
                \item $a$ is not covered by $c$ since $a$ is preferred to $b$ 
                and $b$ is preferred to $c$ 
                \item $a$ is not covered by $d$ since $a$ is preferred to $d$ 
            \end{itemize}
        \end{minipage}
        \caption{}\label{fig:expuc}
    \end{subfigure}

    \begin{subfigure}[b]{0.15\linewidth}
        \centering
        \begin{tikzpicture}[align=center]
            \node [node] (a) at (0,0) {$a$};
            \node [node] (b) at (2,0) {$b$};
            \node [node] (c)  at (2,-2) {$c$};
            \node [node] (d) at (0,-2) {$d$};
        
            \draw [arrow] (a) to [out=15, in=165] node[quotes]  {$3$} (b) ;
            \draw [arrow] (b) to [out=-165, in=-15] node[quotes]  {$2$} (a) ;
        
            \draw [arrow] (a) [out=-30, in=120] to node[quotes,pos=0.2] {$2$} (c) ;
            \draw [arrow] (c) [out=150, in=-60] to node[quotes,pos=0.2] {$3$} (a) ;
            
            \draw [arrow] (a) [out=-75, in=75] to node[quotes] {$4$} (d) ;
            \draw [arrow] (d) [out=105, in=-105] to node[quotes] {$1$} (a) ;
            
            \draw [arrow] (b) [out=-75, in=75] to node[quotes] {$3$} (c) ;
            \draw [arrow] (c) [out=105, in=-105] to node[quotes] {$2$} (b) ;
        
            \draw [arrow] (b) [out=-120, in=30] to node[quotes,pos=0.2] {$3$} (d) ;
            \draw [arrow] (d) [out=60, in=-150] to node[quotes,pos=0.2] {$2$} (b) ;
            
            \draw [arrow] (d) to [out=15, in=165] node[quotes]  {$3$} (c) ;
            \draw [arrow] (c) to [out=-165, in=-15] node[quotes]  {$2$} (d) ;
        \end{tikzpicture}
        \caption{}
    \end{subfigure}%
    ~
    \begin{subfigure}[b]{0.15\linewidth}
        \centering
        \begin{tikzpicture}[align=center]
            \node [node] (a) at (0,0) {$a$};
            \node [node] (b) at (2,0) {$b$};
            \node [node] (c)  at (2,-2) {$c$};
            \node [node] (d) at (0,-2) {$d$};
        
            \draw [arrow] (a) to node[quotes]  {$3$} (b) ;
        
            \draw [arrow] (a) to node[quotes] {$2$} (c) ;
            
            \draw [arrow] (a) to node[quotes] {$3$} (d) ;
            
            \draw [arrow] (b) to node[quotes] {$3$} (c) ;
        
            
        \end{tikzpicture}
        \caption{}
    \end{subfigure}
    ~
    \begin{subfigure}[b]{0.13\linewidth}
        \centering
        \begin{tikzpicture}[align=center, sibling distance=5mm, level distance = 10mm]
            \node [node] (a) at (0,0) {$a$}
                [arrow] child {coordinate (a1) edge from parent node[quotes,pos=0.4] {$3$}}
                child {coordinate (a2) edge from parent node[quotes,pos=0.4] {$2$}}
                child {coordinate (a3) edge from parent node[quotes,pos=0.4] {$3$}};
        
            \node [node] (b) at (1.5,0) {$b$}
                [arrow, {Stealth[length=5pt,round,inset=3pt,width=7pt,flex'=1]}-] child {coordinate (b1) edge from parent node[quotes,pos=0.6] {$3$}};
        
            \node [node] (d) at (0,-1.5) {$d$}
                [arrow, {Stealth[length=5pt,round,inset=3pt,width=7pt,flex'=1]}-] child {coordinate (b1) edge from parent node[quotes,pos=0.6] {$3$}};
        
            \node [node] (c) at (1.5,-1.5) {$c$}
                [arrow, {Stealth[length=5pt,round,inset=3pt,width=7pt,flex'=1]}-] child {coordinate (b1) edge from parent node[quotes,pos=0.6] {$3$}};
        \end{tikzpicture}
        \caption{}\label{fig:localenv}
    \end{subfigure}
    ~
    \begin{subfigure}[b]{0.53\linewidth}
        \centering
        \begin{minipage}[]{1\linewidth}
            $a$ is part of the maximin set because
            \begin{itemize}
                \item $a$ wins at least 2 pairwise comparisons in each head-to-head 
                \item $b$ wins at most 2=5-3 pairwise comparisons in one head-to-head 
                \item $c$ wins at most 2=5-3 pairwise comparisons in one head-to-head 
                \item $d$ wins at most 2=5-3 pairwise comparisons in one head-to-head 
            \end{itemize}
        \end{minipage}
        \caption{}\label{fig:expmm}
    \end{subfigure}

    \caption{From a tournament (a) and a weighted tournament (e), we compute a smallest minimal support for $a\in\uc(G)$ (b) and for $a\in\mm(G)$ (f). We identify and visualize their underlying structure as a $a$-rooted tree (c) or by mean of the out-going edges in the neighborhood of the winning candidate and the in-going edges in those of the losing candidates (g). Finally, we produce textual explanations based on these structures (d) and (h). Examples for the remaining tournament solutions are available in Section~\ref{sec:app_expl} of the Appendix.}
        \label{fig:explanations}
\end{figure*}

\section{Certified User-Friendly Explanations}\label{sec:explanation}

We now inspect MSs in light of the four Gricean principles underpinning effective explanations that we sketched in the introduction: \emph{quantity}, \emph{quality}, \emph{relation} and \emph{manner}~\cite{grice1975logic}.
The formal grounding of MSs ensure that the reasoning is sound, thus covering \emph{quality}.
By definition, MSs provide the minimal information necessary to ensure that a designated candidate is among the election winners. Thus, MSs are informative enough to support our goal and verify both the \emph{quantity} and \emph{relation} principles.
Regarding \emph{manner}, or the way information is delivered, it is possible to create explanations containing the smallest amount of information necessary thanks to our algorithms efficiently computing the smallest MSs for five out of the six studied solutions. 

However, MSs still suffer from two flaws that limit their explanatory power.
First, MSs are (multi-)sets of pairwise comparisons, with a flat structure that makes it hard to gain insights from the unfolding of the tournament solution.
Second, their formal nature renders them more suited to be handled by machines or by domain experts than end-users.
We propose to address these shortcomings by capitalizing on our structural results for MSs to organize the explanation, and by showcasing how to automatically produce intuitive textual explanation from MSs. 



Figure~\ref{fig:explanations} presents two examples of this process.
For path-based rules such as $\uc$, MSs can be presented as a rooted out-tree encompassing all directed paths from the winner to the other candidates (Figure~\ref{fig:tree}), with each path corresponding to a bullet point in the
textual explanation (see Figure~\ref{fig:expuc}).
For myopic-score rules such as $\mm$, MSs can present the neighborhoods of each candidate (see Figure~\ref{fig:localenv}) from which candidates' scores can be bounded to create a compact textual explanation (see Figure~\ref{fig:expmm}).

\section{Conclusions}\label{sec:conclusion}

Minimal supports are compact structures that can serve as certified explanations for tournament solutions. In this paper we have characterized the structure and given the size of smallest minimal supports for well-known tournament solutions, and provided polynomial algorithms for their computation for all but one case. By exploiting our characterizations we have also provided rigorous intuitive visual and textual presentations of a minimal support to demonstrate their applicability as human-oriented explanations.

Minimal supports also have far-reaching connections with several classical problems in computational social choice.
In contrast with previous results on the margin of victory~\cite{brill2020refining,doring2023margin} for the top cycle, the uncovered set, Copeland, and Borda, we prove that the size of a smallest minimal support is the same for every candidate in the winner set, preventing its use as a refinement of those tournament solutions.
From a query complexity perspective, while the top cycle and the uncovered set require $\Omega(\m^2)$ queries~\cite[Theorem 3.5 \& 3.8]{MaitiPalash2024}, we show that, with a perfect oracle, exactly $\m - 1$ queries are needed.
Although finding the smallest minimal explanation is generally hard, we provide algorithms for the top cycle, the uncovered set, Copeland, maximin and Borda with complexity $\mathcal{O}(\m^2\log\n)$, which corresponds to the size of a tournament.

Several key directions merit further exploration: empirically evaluating the usefulness of the presented explanations, studying a broader range of rules, and analyzing structural properties of minimal supports. Finally, in an approach similar to that of \citet{de2000choosing}, minimal supports for different rules on the same tournament could be compared from a set-theoretical perspective.


\section*{Acknowledgments}

The authors thank the reviewers of AAAI26 for their constructive comments and suggestions, which helped improve this paper.

This work is funded by the European Union. Views and opinions expressed are however those of the authors only and do not necessarily reflect those of the European Union or the European Research Council Executive Agency. Neither the European Union nor the granting authority can be held responsible for them. This work is supported by ERC grant 101166894 “Advancing Digital Democratic Innovation” (ADDI).

\begin{figure}[h]
\centering
\includegraphics[height=2cm]{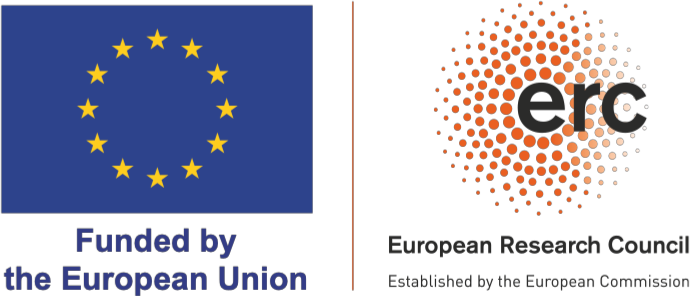} 
\end{figure}

\bibliography{explaintournament}

\newpage
~
\newpage

\appendix

\section*{Appendix}\label{appendix}

\section{Proofs for Section~\ref{sec:sms}}

\subsection{Proofs for Section~\ref{sec:tcuc}}

\tcsize*

\begin{proof}
    Let $\X=(C,E_\X)$ be an MS for $w \in \tc(G)$. 
    Let us now show that every candidate different from $w$ has an in-degree of 1 in $\X$ and that $w$ has an in-degree of 0.
    According to Proposition~\ref{prop:tcstrcut}, for all $c\in\candidates\setminus{\{w\}}$, there exists a directed path $p_c$ from $w$ to $c$ in $\X$.
    Suppose that there exists $c^*\in\candidates\setminus{\{w\}}$ such that there exist distinct $c',c''$ such that $\{(c',c^*),(c'',c^*)\} \subseteq E_\X$.
    Then $p_{c'}\circ(c',c^*)$ and $p_{c''}\circ(c'',c^*)$ are two distinct directed paths from $w$ to $c$ in $\X$. Contradiction.
    Hence, for all $c\in\candidates\setminus{\{w\}}$, every candidate different from $w$ has an in-degree of 1 in $\X$.
    Suppose that there exists $c^*\in\candidates\setminus{\{w\}}$ such that $(c^*,w)\in E_\X$. No $p_c$ contains $(c^*,w)$ otherwise as they start at $w$, they would pass by $w$ twice. 
    Let $\Y=(\candidates,E_\Y)$ be such that $E_\Y=E_\X\setminus{\{(c^*,w)\}}$. 
    For all $c\in\candidates\setminus{\{w\}}$, $p_c$ is a directed path from $w$ to $c$ in $\Y$.
    Thus $w\in\NW{\tc}{\Y}$ but $\Y\subsetneq\X$ so $\X$ is not an MS for $w \in \tc(G)$. Contradiction.
    Hence, $w$ has an in-degree of $0$ in $\X$.
    Finally, since the sum of the in-degree of all nodes is the number of edges in a graph, $|\X|=\m-1$.
\end{proof}

\tcalgo*

\begin{proof}
    With Dijkstra's algorithm~\cite{dljkstra1959note}, it is possible to build directed-rooted tree in $G$ with $w$ in $\mathcal{O}(|E|+|C|)$, i.e., $\mathcal{O}(\m^2)$. The shortest-path property is ensured by the fact that Dijkstra's algorithm minimizes the cost of the directed path to every candidate. (Since we are working with unweighted graphs, simply exploring the graph in breadth first search will work). Hence, for every candidate $c\in\candidates\setminus{\{w\}}$ such that $\mu(w,c) = 1$, $c$ will be a child of $w$ in the obtained directed-rooted tree.
\end{proof}

\subsection{Proofs for Section~\ref{sec:wuc}}

\wucstrcut*

\begin{proof}
    As this proof follows the same idea used in the proof of Proposition~\ref{prop:tcstrcut}, we only provide a sketch of the proof.
    Let $\X=(\candidates,\mu_\X)$ be an MS for $w \in \wuc(G)$.
    If exists $c\in\candidates\setminus{\{w\}}$, $\mu_\X(w,c)<\lceil(\n+1)/2\rceil$ or for all $c'\in\candidates$, $\mu_\X(w,c')+\mu_\X(c',c)<\n+1$, we prove there exists a completion $\X'$ of $\X$ such that none of the conditions of Lemma~\ref{lem:wucchara} is satisfied for $c$. Hence, $w\not\in\mm(\X')$, $w\not\in\NW{\mm}{\X}$ and $\X$ is not an MS for $w \in \wuc(G)$. Contradiction.
\end{proof}

\wucsizehard*

\begin{proof}
    The inclusion in NP is clear. To prove NP-hardness, we proceed with a reduction from the NP-complete \textsc{Set Cover} problem \cite{karp1972reducibility}. Given a set of $p$ element $U=\{e_1,e_2,\dots,e_p\}$ and a set $S$ of $q$ subsets of $U$, $S=\{s_1,s_2,\dots,s_q\}$ such that $\bigcup_{s\in S} s = U$, and an integer $k$, the set cover problem looks at the existence of a subset $D\subseteq S$ of size at most $k$ such that $\bigcup_{s\in D} s = U$. 
    With this notation, we now prove that there exists a set cover of size less than $k\in\mathbb{N}$ if and only if there exists an MS of size less than $p+q+k$ for a specific tournament.
    
    Let $U$, $S$, $k$ be an instance of set cover. First, we can safely assume 
    that there exists no $s\in S$ such that $S=U$.
    %
    We define a complete 2-weighted tournament $G(U,S)$ ($G$ for short) with the following $p+q+1$ candidates: a candidate $w$ that is the winner for whom we compute an SMS, $p$ candidates $c_{e_j}$, one for each element of $U$, and $q$ candidates $c_{s_j}$ for the elements of $S$. We define the weight function $\mu$ according to the adjacency matrix described in Table~\ref{tab:adjacencyNP}.

    \begin{table}[b]
    \centering
    \begin{tabular}{c|c|c|c|}
        \cline{2-4}
        & $w$ & $c_{e_1}\dots c_{e_p}$ & $c_{s_1}\dots c_{s_q}$ \\ 
        \hline
        \multicolumn{1}{|c|}{$w$} & 0 & 0 & 2 \\ 
        \hline
        \multicolumn{1}{|c|}{\makecell[c]{$c_{e_1}$\\\vdots\\$c_{e_p}$}} & 2 & 1 & \makecell[c]{if $e_i \in s_j$ \\ then 1 \\ else 2} \\ 
        \hline
        \multicolumn{1}{|c|}{\makecell[c]{$c_{s_1}$\\\vdots\\$c_{s_q}$}} & 0 & \makecell[c]{if $e_j \in s_i$ \\ then 1 \\ else 0} & 1 \\ 
        \hline
    \end{tabular}
    \caption{Adjacency matrix of tournament $G(U,S)$ associated to a set cover instance. Note that, since $\mu(c,c) = 0$, there are zeros on the diagonal that we did not represent.}
    \label{tab:adjacencyNP}
    \end{table}

    Observe that $w \in \wuc(G)$ because $\forall s\in S$, $w$ is not covered by $c_s$ because $\exists s'\neq s$ such that $\mu(w,c_{s'})=2>1=\mu(c_s,c_{s'})$ and $\forall e \in U$, $w$ is not covered by $c_e$ because $\exists s \in S$ such that $e \not \in s$ and $\mu(w,c_{s})=2>1=\mu(c_e,c_{s})$.

    We now show that if there exists a set cover of $(U,S)$ of size less than $k$ then there exists an MS of $G(U,S)$ of size less than $p+q+k$
    Suppose there exists a set cover $D$ of $(U,S)$ of size less than $k$.
    Let $\X\subseteq G$ such that: $\forall s \in D$, $\mu_\X(w,s) = 2$, $\forall s \in S\setminus{D}$, $\exists! s'\in D$ such that $\mu_\X(s',s) = 1$ and $\forall e \in U$, $\exists! s'\in D$ such that $\mu_\X(s',e) = 1$. $|\X|=2|D|+1(p+q-|D|)=p+q+|D|\leq p+q+k$.
    $\forall s \in D$, $\mu_\X(w,c_s) = 2\geq\left\lceil(2+1)/2\right\rceil$.
    $\forall s \in S\setminus{D}$, $\exists s'\in D$ such that $\mu_\X(c_{s'},c_s) + \mu_\X(w, c_{s'}) = 1 + 2 \geq 2 + 1$.
    $\forall e \in U$, since $D$ is a cover set of (U,S), $\exists s\in D$ such that $\mu_\X(c_s,c_e) + \mu_\X(w, c_s) = 1 + 2 \geq 2 + 1$. 
    Hence, $w\in\NW{\wuc}{\X}$ and there exists an MS contained in $\X$, i.e., of size at most $p+q+k$.

    We now show the converse, if there exists an MS of $G(U,S)$ of size less than $p+q+k$ then there exists a set cover of $(U,S)$ of size less than $k$.
    Suppose there exists an MS $\X$ for $w\in \wuc(G)$ such that $|\X|\leq p+q+k$.
    Let $D=\{c:c\in\candidates\setminus{\{w\}}, \mu(w,c)\geq \left\lceil(2+1)/2\right\rceil\}$ be the set of dominion of $w$.
    By minimality of an MS, 
    $\forall c\in D$, $\forall c' \in \candidates\setminus{\{w\}}$, $\mu_\X(c',c) = 0$ and 
    $\forall c\in \candidates\setminus{D}$, $\exists! c'\in D$, $\mu_\X(c',c) = 1$ and $\forall c'' \in \candidates\setminus{\{c'\}}$, $\mu_\X(c'',c) = 0$.
    Hence, $|\X|=2|D|+1(|\candidates|-1-|D|)=2|D|+1(p+q-|D|)= p+q+|D|$ and since $|\X|\leq p+q+k$, we have $|D|\leq k$.
    
    Since $\mu_X\leq \mu$ and $\forall e \in U$, $\mu(w,c_e)=0$, $\forall e \in U$, we have $\exists c \in \candidates\setminus{\{w\}}$ such that $\mu_X(c,c_e)=1$ and $\mu_X(w,c)=2$. 
    However, $\forall e \in U$, $\forall c \in \candidates$, $\mu(c,c_e)\geq1$ and $\mu(w,c)\geq2$ if and only if $\exists s \in S$ such that $c=c_s$ and $e\in s$.
    Thus, $\forall e \in U$, $\exists s \in S$ such that $e\in s$ and $s\in D$, i.e., $D$ is a set cover of $(U,S)$.
    Hence, there exists a set cover of $(U,S)$ of size less than $k$.
\end{proof}

\wucbounds*

\begin{proof}
    We start with the lower bound.
    Let $\X=(\candidates,\mu_\X)$ be an MS for $w \in \wuc(G)$,
    Thanks to Proposition~\ref{prop:wucstrcut}, we know that for all losing candidates $c\in\candidates\setminus{\{w\}}$, $\mu_\X(w,c)\geq\lceil(\n+1)/2\rceil$ or exists a unique $c'$ such that $\mu_\X(w,c')+\mu_\X(c',c)\geq\n+1$. 
    In particular, we have $\max_{c''\neq w} \mu_\X(w,c'')+\mu_\X(c',c)\geq\n+1$ which implies $\mu_\X(c',c)\geq\n+1-\max_{c''\neq w} \mu_\X(w,c'')$.
    Let $D =\{c:c\in\candidates, \mu_\X(w,c)\geq\lceil(\n+1)/2\rceil\}$ and $k=|D|$.

    \smallskip
    \noindent If $k=0$, 
    \begin{align*}
        |\X| & = \sum_{c\neq w} \mu_\X(w,c) + \sum_{c\neq w}  \sum_{c'\neq w} \mu_\X(c',c)\\
        & \geq \max_{c\neq w} \mu_\X(w,c) + (\m-1)(\n+1-\max_{c\neq w} \mu_\X(w,c))\\
        & = \n +1 + (\m-2)(\n+1-\max_{c\neq w} \mu_\X(w,c))\\
        & \geq \n +1 + \m-2
    \end{align*}
    If $k\geq1$, 
    \begin{align*}
        |\X| & = \sum_{c\neq w} \mu_\X(w,c) + \sum_{c\neq w} \sum_{c'\neq w} \mu_\X(c',c)\\
        & = \sum_{c\in D} \mu_\X(w,c) + \sum_{c\not\in D\cup\{w\}} \sum_{c'\neq w} \mu_\X(c',c)\\
        & \geq \max_{c\neq w} \mu_\X(w,c) + (k-1)\lceil(\n+1)/2\rceil \\
        & + (\m-1-k)(\n+1-\max_{c\neq w} \mu_\X(w,c))\\
        & = \n+1 + (k-1)\lceil(\n+1)/2\rceil\\
        & + (\m-2-k)(\n+1-\max_{c\neq w} \mu_\X(w,c))\\
        & = \n+1 + (\m-3)(\n+1-\max_{c\neq w} \mu_\X(w,c))\\
        & + (k-1)(\max_{c\neq w} \mu_\X(w,c) + \lceil(\n+1)/2\rceil-\n-1)
    \end{align*}
    Note that $\n+1-\max_{c\neq w} \mu_\X(w,c) \geq 1$ and since, $\max_{c\neq w} \mu_\X(w,c) \geq \left\lceil(\n+1)/2\right\rceil$, $\max_{c\neq w} \mu_\X(w,c) + \lceil(\n+1)/2\rceil\geq \n + 1$. Hence, $|\X| \geq \n+\m-2$.
    
    We prove that the bound is tight.
    Given $\n$ voters, a complete $\n$-weighted tournament $G =(\candidates,\mu)$ with $|\candidates| = \m$, and a candidate $w \in \candidates$ such that there exist a candidate $c\in\candidates\setminus{\{w\}}$ such that $\mu(w,c)=\n$ and for every candidate $c'\in\candidates\setminus{\{w,c\}}$, $\mu(c,c')\geq 1$.
    For all $c'\in\candidates\setminus{\{w\}}$, exists $c\in\candidates\setminus{\{w\}}$ such that $\mu(w,c)+\mu(c,c')=\n+1$. Hence, $w\in\wuc(G)$.
    Then let $\X=(\candidates,\mu_\X)$ such that $\mu(w,c)=\n$ and for every candidate $c'\in\candidates\setminus{\{w,c\}}$, $\mu(c,c') = 1$ and $\mu_X = 0$ everywhere else. 
    $|\X| = \n + \m - 2$ and for all completions $\X'$ of $\X$, $\mu_{\X'}(w,c)=\n\geq\lceil(\n+1)/2\rceil$ and for every candidate $c'\in\candidates\setminus{\{w,c\}}$, $\mu_{\X'}(w,c')+\mu_{\X'}(c,c'')=\n+1$.

    For the upper bound, thanks to Proposition~\ref{prop:wucstrcut}, we know that $\n+1$ pairwise comparisons are enough to prove that $w$ cover one losing candidate.
    This gives us a global upper bound of $(\m-1)(\n+1)$ which is tight.
    Given $\n$ voters, a complete $\n$-weighted tournament $G =(\candidates,\mu)$ with $|\candidates| = \m$, and a candidate $w \in \candidates$ such that for every candidate $c\in\candidates\setminus{\{w\}}$ such that $\mu(w,c)=1$ and exists a unique candidate $d_c$ such that $\mu(d_c,c)=0$. 
    Hence, for all $c\in\candidates\setminus{\{w\}}$, for all $c'\in\candidates\setminus{\{w,d_c\}}$, $\mu(w,c)<\lceil(\n+1)/2\rceil$, $\mu(w,c)+\mu(c,c')=\mu(w,c)+\n-\mu(c',c)\leq 1+\n-1<\n+1$ and $\mu(w,c)+\mu(c,d_c)=\mu(w,c)+\n-\mu(d_c,c)= \n+1$.
    Hence, $w\in\wuc(G)$ and the only possible MS $\X=(\candidates,\mu_X)$ is for every candidate $c\in\candidates\setminus{\{w\}}$, $\mu(w,c)=1$, $\mu(d_c,c)=n$ and $\mu_X = 0$ everywhere else. $|\X|=(\m-1)(1+\n)$ and since it is the only MS, it is also an SMS.
\end{proof}

\subsection{Proofs for Section~\ref{sec:mm}}

\mmchara*

\begin{proof}
    Building an MS for $w \in \mm(G)$ is rather straightforward. We simply need to ensure that $w$ wins a minimum number $t$ of pairwise comparisons in each head-to-head with another candidate and ensure that each losing candidate loses at least $n-t$ pairwise comparisons against another candidate. 
    Suppose we do not do that and let $\X=(\candidates,\mu_\X)$ be an MS for $w \in \mm(G)$. 
    Suppose for all $t\leq\left\lceil\n/2\right\rceil$ exists $c_0\in\candidates\setminus{\{w\}}$ such that $\mu_\X(w,c_0)< t$ and for all $c'\in\candidates$, $\mu_\X(c',c_0)< \n-t$.
    Then let $\X'$ be a completion of $\X$ such that $\mu_{\X'}(w,c_0)=\mu_\X(w,c_0)$ and for all $c'\in\candidates$, $\mu_{\X'}(c',c_0)=\mu_\X(c',c_0)$. 
    Then the maximin score of $w$ is $\sigma_w \mm(w,\X')\geq \mu_{\X'}(w,c_0)< t$ while the maximin score of $c_0$ is $\sigma_{c_0} \mm(c_0,\X')\geq \n - \max_{c'\in \candidates} \mu_\X(c',c_0) > \n - (\n - t) > t$. Since $\sigma_{\mm}(c_0,\X')>\sigma_{\mm}(w,\X')$, $w\not\in\NW{\mm}{\X}$. Contradiction.
\end{proof}

\mmbounds*

\begin{proof}
    The lower bound is achieved by maximizing the values of $t$ and $k=|\{c:c\in\candidates,\,\mu(w,c)\geq \n-t\}|$. Typically, if for every candidate $c\in\candidates\setminus{\{w\}}$, $\mu(w,c)=\n$, we have $t=\left\lfloor\n/2\right\rfloor$ and $k=\m-1$ and $|\X|=(\n-\left\lfloor\n/2\right\rfloor)(\m-1)+\n(\m-1-(\m-1)))=\left\lceil\n/2\right\rceil(\m-1)$.
    Conversely, if $k=0$, $|\X|=(\n-t)0+\n(\m-1-0))=\n(\m-1)$.
\end{proof}

\mmalgo*

\begin{proof}
    Following the ideas developed in Theorem~\ref{th:mmsize}, we introduce Algorithm~\ref{alg:findSAXpMM}. The SMS returned by the algorithm maximizes the maximin score of $w$, unless its score is larger than $\left\lfloor\n/2\right\rfloor$, because ensuring a larger score would break the subset minimality of the SMSs. Hence, the maxwin property is satified.
\end{proof}

\begin{algorithm}
\caption{\textsc{findMaxwinSMS-MM} 
}\label{alg:findSAXpMM}
\KwData{Complete $n$-weighted tournament $G =(\candidates,\mu)$, winning candidate $w \in \mm(G)$}
\KwResult{a maxwin-SMS $\X=(\candidates,\mu_{\X})$}
$\mu_{\X}(\cdot,\cdot) \gets 0$\;
$t \gets \min(\min_{c\neq w}\mu(w,c),\left\lfloor\n/2\right\rfloor)$\;
\For{$c \in \cstar$}{
    \eIf{$\mu(w,c)\geq \n -t$}{
        $\mu_{\X}(w,c) \gets \n -t$\;
    }{
        $\mu_{\X}(w,c) \gets t$\;
        \For{$c' \in \cstar$}{
            \If{$\mu_{\X}(c',c)\geq \n-t$}{
                $\mu_{\X}(c',c) \gets \n - t$\;
                break\;
            }
        }
    }
}
\Return{$\X$}
\end{algorithm}

\subsection{Definitions and Proofs for Section~\ref{sec:borda}}


We start this section by presenting useful results characterizing necessary winners in terms of the score margins of the winning candidate, a notion that will be used to construct our algorithms for the computation of minimal supports. Our results are for weighted tournaments and the Borda rule, but they also apply to Copeland by setting $\n=1$.

We also define the notions of \emph{out-going edges} of a candidate $x$ in $E$ as $\edge{E}{x}{} = \{(x,y) : y \in \candidates,\, (x, y) \in E\}$ and its \emph{in-going edges} as $\edge{E}{}{x} = \{(y,x): y \in \candidates,\, (y, x) \in E\}$.

\begin{example}
Consider tournament $G=(\candidates,E)$ in Figure~\ref{fig:exunw}.
In our notation, we have: $\edge{E}{a}{}=\{(a,b),(a,d)\}$, $\edge{E}{b}{}=\{(b,c),(b,d)\}$, $\edge{E}{}{b}=\{(a,b)\}$ and $\edge{E}{}{a}=\{(c,a)\}$.
\end{example}

We will also use an equivalent definition that presents weighted tournaments as multigraphs.
A \emph{multiset} is a pair $(A,\mu)$ where $A$ is the underlying set
and $\mu: A \to \mathbb{N}$ a multiplicity function giving the number of occurrences of the elements in $A$. 
The cardinality $|(A,\mu)|$ is defined as $\sum_{e \in A} \mu(e)$. 
For multisets such as $(\{e_0,e_1,\dots,e_n\},\mu)$, we will adopt the common notation $\{e_o^{\mu(e_0)},e_1^{\mu(e_1)},\dots,e_n^{\mu(e_n)}\}$ and we will omit the superscript in case of a multiplicity of $1$.
A \emph{multigraph} is a pair $G = (V, E)$ where $V$ is a nonempty finite set of vertices and $E$ a multiset of elements of $V \times V$.
A partial $n$-weighted tournament $G = (\candidates, \mu)$ can naturally be represented as a multigraph $(\candidates,(\candidates\times\candidates,\mu))$, simply denoted $(\candidates, E)$ where $E$ is now a multiset.
The (multi)sets of in-going and out-going edges are naturally defined for multigraph.
In our figures we will omit edges with a weight/multiplicity of $0$.

\begin{example}
Consider the weighted tournament $G_w=(\candidates,E_w)$ in Figure~\ref{fig:exw}.
In multiset notation we have: 
$\edge{E_w}{a}{}{=}\{(a,b)^3,(a,c)^3,(a,d)^5\}$ and $\edge{E_w}{}{a}=\{(b,a)^2,(c,a)^2\}$.
\end{example}w

Let us denote the Borda score of candidate $c$ for the tournament $G$ as $\sigma_{\borda}(c,G)$. When the context will make the considered tournament clear, we will shorten this notation to $\sigma_{\borda}(c)$.

\begin{definition}
    Given $n$ voters, a partial $n$-weighted tournament $G = (\candidates, \mu)$, and a candidate $c \in \candidates$, we call $\sigma_{\borda}^{min}(c,G) = min_{G' \in [G]} \sigma_{\borda}(c,G')$ (resp. $\sigma_{\borda}^{max}(c,G) = max_{G' \in [G]} \sigma_{\borda}(c,G')$) the \emph{minimal} (resp. \emph{maximal}) \emph{score} that candidate $c$ can obtain in any completion of~$G$.
\end{definition}

We then give a formula to compute $\sigma_{\borda}^{min}$ and $\sigma_{\borda}^{max}$ in partial tournaments.

\begin{lemma}\label{lem:scorebord}
    Given $n$ voters, a partial $n$-weighted tournament $G=(\candidates,E)$ and a candidate $c\in\candidates$, $\sigma_{\borda}^{min}(c,G) = |\edge{E}{c}{}|$ and $\sigma_{\borda}^{max}(c,G) = n(|\candidates| - 1) - |\edge{E}{}{c}|$.
\end{lemma}
\begin{proof}
    $\sigma_{\borda}^{min}(c,G) = |\edge{E}{c}{}|$ as the minimum number of pairwise comparisons to be won by $c$ is the number of pairwise comparisons already won by $c$. 
    $\sigma_{\borda}^{max}(c,G) = n(|\candidates| - 1) - |\edge{E}{}{c}|$ as the maximum number of pairwise comparisons to be won  by $c$ corresponds to the total number of pairwise comparisons to be played by $c$ less those already lost.
\end{proof}

Then, we can introduce the notion of score margin, the main notion that we will used in the proofs of this section.

\begin{definition}
    Given $n$ voters, a partial $n$-weighted tournament $G = (\candidates, E)$, a pair of candidates $c,c' \in \candidates$, the \emph{score margin} of candidate $c$ w.r.t. $c'$ for $G$ is $\smo{c}{c'}{G}{\borda} = \sigma_{\borda}^{min}(c,G) - \sigma_{\borda}^{max}(c',G)$. 
\end{definition}

From now on we will consider the score margins of the winning candidate $w$ for the Borda rule, writing $\sm{c}{G}$ as a shorthand for $\smo{w}{c}{G}{\borda}$.

We can then provide a characterisation for necessary winner in terms of score margin, in line with an analogous result by \citet{konczak_voting_2005} in voting.

\begin{proposition}\label{prop:charaNW}
    Given $n$ voters, a partial $n$-weighted tournament $G =(\candidates,E)$, and a candidate $w\in \candidates$, $w \in \NW{Borda}{G}$ if and only if for all $c \in \cstar$, $\sm{c}{G} \geq 0$.
\end{proposition}

\begin{proof}
    \fbox{$\impliedby$} is straightforward.
    \fbox{$\implies$} To prove this direction, for each losing candidate $c$, we build a completion $G_c$ of $G$ where $\sigma_{\borda}(w,G_c) = \sigma_{\borda}^{min}(w,G)$ and $\sigma_{\borda}(c,G_c) = \sigma_{\borda}^{max}(c,G)$. Since $w$ is a necessary winner, $w$ is a winner of $G_c$. Hence $\sigma_{\borda}(w,G_c) \geq \sigma_{\borda}(c,G_c)$ i.e. $\sm{c}{G} \geq 0$.
    To build $G_c$ we proceed as in the proof of Lemma~\ref{lem:scorebord} and complete $G$ such that $w$ loses all remaining contests while $c$ wins them all.
    Let $E_c = E\cup\{(x,w)^{\n-\mu(w,x)}: x \in \candidates\}\cup\{(c,y)^{\n-\mu(y,c)}: y \in \candidates\}$, complete $E_c$ (in any way) to get $E_c'$. $G_c$ is the tournament $(\candidates,E_c')$.
\end{proof}

\bordasize*

For ease of understanding, we have divided the proof into several lemmas.
First, we show that there are three necessary conditions imposing lower bounds on the size of any MS (Lemmas~\ref{lem:lowerbound1} to \ref{lem:lowerbound3}). Then we prove that depending on the specific case, one of those three lower bounds is matched (Lemmas~\ref{lem:case1} to \ref{lem:case3}).

We start by showing a first necessary condition.

\begin{lemma}\label{lem:lowerbound1}
    Given $n$ voters, a partial \mbox{$n${-}weighted} tournament $G =(\candidates,E)$ with $|\candidates| = \m$, and a winning candidate $w \in \borda(G)$ with a Borda score of $\sigma_w$. For all MSs $\X$ for $w \in \borda(G)$, we have $|\X|\geq(m-1)(n(m-1)-\sigma_w)$.
\end{lemma}

\begin{proof}
    Let $\X$ be an MS for $w \in \borda(G)$, by definition $\X$ is an MS for $w \in \borda(G)$.
    Summing the inequalities from Proposition~\ref{prop:charaNW} we have:
    \begin{align*}
        & \sum_{c \neq w} \sm{c}{G} \geq 0\\
        \implies & \sum_{c \neq w} \left( \sigma_{\borda}^{min}(w,G) - \sigma_{\borda}^{max}(c,G)\right) \geq 0 \\
        \implies & (\m-1) |\edge{\X}{w}{}| - \sum_{c \neq w} (\n(\m - 1) - |\edge{\X}{}{c}|) \geq 0 \\
        \implies & (\m-1) |\edge{\X}{w}{}| - (\m-1)\n(\m - 1) + \sum_{c \neq w} |\edge{\X}{}{c}| \geq 0 \\
        \implies & \sum_{c \neq w} |\edge{\X}{}{c}| \geq (\m-1)(\n(\m - 1)-|\edge{\X}{w}{}|) \\
        \implies & |\X| \geq (\m-1)(\n(\m - 1)-|\edge{\X}{w}{}|)
    \end{align*}     
    To conclude, note that $\edge{\X}{w}{} \subseteq \edge{E}{w}{}$ thus $|\edge{\X}{w}{}| \leq |\edge{E}{w}{}|=\sigma_{\borda}(w)$, obtaining the desired bound.
\end{proof}

Two other necessary conditions are derived from the characterization of Proposition~\ref{prop:charaNW}. The first one appears when we look at the expected value of the losers' scores and the second one when we consider the maximum of the losers's scores.

\begin{lemma}\label{lem:lowerbound2}
    Given $n$ voters, a partial \mbox{$n${-}weighted} tournament $G =(\candidates,E)$ with $|\candidates| = \m$, and a winning candidate $w \in \borda(G)$. For all MSs $\X$ for $w \in \borda(G)$, we have $|\X| \geq \lceil \n(\m - 1)^2/\m \rceil$.
\end{lemma}

\begin{proof}
    Let $\X$ be an MS for $w \in \borda(G)$. For $w$ to be a necessary winner, according to Proposition~\ref{prop:charaNW}, we need $\forall c \in \cstar,\, \sigma^{min}_{\borda}(w,\X) \geq  \sigma^{max}_{\borda}(c,\X)$. This gives us the constraint $\sigma^{min}_{\borda}(w,\X) \geq  \mathbb{E}_{c\neq w}\sigma^{max}_{\borda}(c,\X)$. Since there always exists a losing candidate with score at least as high as the mean of the losing candidates, the winner's score must be larger than the average of the losers' scores. From this inequality we can derive a second lower bound on the size of an SMS.
    
    \begin{align*}
        & \sigma_{\borda}^{min}(w,G) \geq \mathbb{E}_{c\neq w}\sigma^{max}_{\borda}(c,\X)\\
        \implies & \sum_{c\neq w}\mu_{\X}(w,c)\geq\mathbb{E}_{c\neq w}(\n(\m-1) - \mu_{\X}(w,c))\\
        \implies & \sum_{c\neq w}\mu_{\X}(w,c)\geq \frac{1}{\m-1}(\n(\m-1)^2 - \sum_{c\neq w}\mu_{\X}(w,c))\\
        \implies & (\m-1)\sum_{c\neq w}\mu_{\X}(w,c)\geq \n(\m-1)^2 - \sum_{c\neq w}\mu_{\X}(w,c)\\
        \implies & \m\sum_{c\neq w}\mu_{\X}(w,c)\geq \n(\m-1)^2\\
        \implies & \sum_{c\neq w}\mu_{\X}(w,c)\geq \frac{\n(\m-1)^2}{m}\\
        \implies & \sum_{c\neq w}\mu_{\X}(w,c)\geq \left\lceil\frac{\n(\m-1)^2}{m}\right\rceil\\
    \end{align*}
    To conclude, note that 
    
    \noindent $\sum_{c\neq w}\mu_{\X}(w,c) \leq \sum_{c'\in\candidates}\sum_{c\in\candidates}\mu_{\X}(c',c)= |\X|$.
\end{proof}

\begin{lemma}\label{lem:lowerbound3}
    Given $n$ voters, a partial \mbox{$n${-}weighted} tournament $G =(\candidates,E)$ with $|\candidates| = \m$, and a winning candidate $w \in \borda(G)$. For all MSs $\X$ for $w \in \borda(G)$, we have $|\X| \geq \n(\m-1) - \min_{c\neq w}\mu(w,c)$.
\end{lemma}

\begin{proof}
    Let $\X$ be an MS for $w \in \borda(G)$. For $w$ to be a necessary winner, according to Proposition~\ref{prop:charaNW}, we need $\forall c \in \cstar,\, \sigma^{min}_{\borda}(w,\X) \geq  \sigma^{max}_{\borda}(c,\X)$.
    We have $m-1$ local constraints (one for each losing candidate) which can be combined as $\sigma^{min}_{\borda}(w,\X) \geq  \max_{c\neq w}\sigma^{max}_{\borda}(c,\X)$. This transforms into $|\X| \geq \sum_c \mu_{\X}(w,c) \geq n(m-1) - \min_{c\neq w}(\mu_{\X}(w,c))$. Finally, since $\mu_{\X} \leq \mu$, $|\X| \geq n(m-1) - \min_{c\neq w}(\mu(w,c))$.
\end{proof}

Now that we have our three lower bounds, we can now prove that there are met in different cases.

\begin{lemma}\label{lem:case1}
    Given $n$ voters, a partial \mbox{$n${-}weighted} tournament $G =(\candidates,E)$ with $|\candidates| = \m$, and a winning candidate $w \in \borda(G)$ with a Borda score of $\sigma_w$. If $\sigma_w \leq n(m-1)-\min_{c\neq w}\mu(w,c)$, there exists an MS $\X$ for $w \in \borda(G)$ such that $|\X|{=}(m-1)(n(m-1)-\sigma_w)$.
\end{lemma}

\begin{proof}
    Suppose $\sigma_w \leq n(m-1)-\min_{c\neq w}\mu(w,c)$.
    Hence, for all $c\in\cstar$, $\sm{c}{\edge{E}{w}{}}=\sigma_{\borda}^{min}(c,G) - \sigma_{\borda}^{max}(c',G) = \sigma_w - (n(m-1) - \mu(w,c)) \leq 0$ and $\edge{E}{w}{}$ is at best an MS if all deltas are null else it is a subset of an MS. 
    We start with $\X = \edge{E}{w}{}$.
    The idea is then to simply add defeats to candidates that can still achieve more victories than the tournament winner, i.e., for each $c\in\cstar$, add elements of $\edge{E}{}{c}$ until $\sm{c}{\edge{E}{w}{}}=0$.
    In the resulting minimal support $\X$, each losing candidate would have $(n(m-1)-\sigma_w)$ defeats among which all the tournament winner victories can be found.
    Thus, $|\X|{=}(m-1)(n(m-1)-\sigma_w)$.
\end{proof}

Before working we the other cases, we show the existence of an SMS which is either included in $\edge{E}{w}{}$ or contains it.

\begin{lemma}\label{lem:saxpstruct}
    Given $n$ voters, an $n$-weighted tournament $G =(\candidates,E)$ with $|\candidates| = \m$, and a winning candidate $w\in\borda(G)$, there exists an SMS $\X$ for $w\in \borda(G)$ such that $\X \subseteq \edge{E}{w}{}$ or $\edge{E}{w}{} \subseteq \X$.
\end{lemma}

\begin{proof}
    The idea is to go from one SMS to another by adding a win in pairwise comparison for $w$ and removing a lose for another candidate. This can be done as long as there are win for $w$ not in the SMS and defeats for other candidates in the SMS. When one of those two conditions is not met anymore we stop.
    First we show that:    
    (a) For all SMSs $\X$ for $w \in \borda(G)$ such that $\exists e_{\wbar} \in E\setminus{\edge{E}{w}{}}$ and $\exists e_{w} \in \edge{E}{w}{}$ with $e_{\wbar} \in \X$ and $e_{w} \not \in \X$, $\left(\X\setminus{\{e_{\wbar}\}}\right) \cup \{e_{w}\}$ is another SMS for $w \in \borda(G)$. 
    Let  $\X' = \left(\X\setminus{\{e_{\wbar}\}}\right) \cup \{e_{w}\}$. Let $c_0\in \candidates$ s.t. $e_{\wbar} \in \edge{E}{c_0}{}$.
    We have $\sigma^{min}(c,\X') = \sigma^{min}(c,\X) + 1$, for all $c \in \candidates\setminus{\{w,c_0\}}$, $\sigma^{max}(c,\X) = \sigma^{max}(c,\X')$ and $\sigma^{max}(c_0,\X) = \sigma^{max}(c,\X')+1$.
    All in all, for all $c \in \cstar$, $\sm{c}{\X'} \geq \sm{c}{\X}$. Hence, $\X'$ is a weak MS for $w \in \borda(G)$. 
    Additionally, $|\X|=|\X'|$ and $\X$ is an SMS for $w \in \borda(G)$. Thus, $\X'$ is one too.
    
    Finally, given any SMS $\X$ of $w \in \borda(G)$, applying (a) $\min(|\X\setminus{\edge{E}{w}{}}|,|\edge{E}{w}{}\setminus{\X}|)$ times, we obtain an SMS for $w \in \borda(G)$ such that $\X \subseteq \edge{E}{w}{}$ or $\edge{E}{w}{} \subseteq \X$.
\end{proof}

Now, we are ready to attack the other case. We will start by dividing it into two sub-cases.

\begin{lemma}\label{lem:subcases}
    Given $n$ voters, a partial \mbox{$n${-}weighted} tournament $G =(\candidates,E)$ with $|\candidates| = \m$, and a winning candidate $w \in \borda(G)$ with a Borda score of $\sigma_w$. For all SMSs $\X$ for $w \in \borda(G)$ such that $\X \subseteq \edge{E}{w}{}$, we have $\forall c \in \cstar,\, \delta(c,\X) = 1 \lor \exists c_0 \in \cstar,\, \delta(c_0,\X) \leq 0$.
\end{lemma}

\begin{proof}
    Let $\X$ be an SMS for $w \in \borda(G)$. 
    Every element of $\edge{E}{w}{}$ contributes 1 to every delta (as a victory for the tournament winner, it increases its minimal possible score by 1) and 1 to the specific candidate $c$ it is pointing toward ((w,c) as defeat for c decreases its maximal possible score by 1). Hence, removing an element of $\edge{E}{w}{}$ reduces one delta by 2 and the others by 1.
    By subset minimality its is not possible to remove any extra elements from $\X$ either when there is no delta greater than 2 or a null delta.
\end{proof}

We start with the first sub-case.

\begin{lemma}\label{lem:case2}
     Given $n$ voters, a partial \mbox{$n${-}weighted} tournament $G =(\candidates,E)$ with $|\candidates| = \m\geq2$, and a winning candidate $w \in \borda(G)$ with a Borda score of $\sigma_w$. If $\sigma_w > \n(\m-1)-\min_{c\neq w}\mu(w,c)$ and $\forall c \in \cstar,\, \delta(c,\X) \leq 1$, there exists an MS $\X$ for $w \in \borda(G)$ such that $|\X| = \lceil \n(\m - 1)^2/\m \rceil$.
\end{lemma}

\begin{proof}
    Lemma~\ref{lem:saxpstruct} gives us the existence of an SMS $\X$ for $w \in \borda(G)$ such that $\X \subseteq \edge{E}{w}{}$.
    Suppose $\forall c \in \cstar,\, \delta(c,\X) \leq 1$.
    \begin{align*}
        & \sum_{c \neq w} \left( \sigma_{\borda}^{min}(w,\X) - \sigma_{\borda}^{max}(c,\X)\right) \leq \sum_{c \neq w} 1\\
        \implies & \sum_{c \neq w} \left( \sum_{c'\in\candidates} \mu_{\X}(w,c') - \left(\n(\m-1-\sum_{c'\in\candidates} \mu_{\X}(c',c)\right)\right)\\ 
        & \leq \m-1\\
        \implies & \sum_{c \neq w} \left( \sum_{c'\in\candidates} \mu_{\X}(w,c') - \left(\n(\m-1- \mu_{\X}(w,c)\right)\right) \leq \m-1\\
        \implies & (\m-1) \sum_{c \neq w} \mu_{\X}(w,c) - \n(\m - 1)^2 \\
        & + \sum_{c \neq w} \mu_{\X}(w,c) \leq \m-1\\
        \implies & \m \sum_{c \neq w} \mu_{\X}(w,c) \leq \n(\m - 1)^2 + \m-1\\
        \implies & \sum_{c \neq w} \mu_{\X}(w,c) \leq \frac{\n(\m - 1)^2 + \m-1}{\m}
    \end{align*}
    Observe that the left hand-side of the last equality is an integer and since $(\m-1)/\m$ is strictly less than 1 and non zero ($\m\geq2$), $\sum_{c \neq w} \mu_{\X}(w,c) \leq \n(\m - 1)^2/\m + (\m-1)/\m \implies |\X| = \sum_{c \neq w} \mu_{\X}(w,c) \leq \lceil \n(\m - 1)^2/\m \rceil$.
\end{proof}

We continue with the second sub-case.

\begin{lemma}\label{lem:case3}
    Given $n$ voters, a partial \mbox{$n${-}weighted} tournament $G =(\candidates,E)$ with $|\candidates| = \m$, and a winning candidate $w \in \borda(G)$ with a Borda score of $\sigma_w$. If $\sigma_w > \n(\m-1)-\min_{c\neq w}\mu(w,c)$ and $\exists c_0 \in \cstar,\, \delta(c_0,\X) \leq 0$, there exists an MS $\X$ for $w \in \borda(G)$ such that $|\X| \leq \n(\m-1) - \min_{c\neq w} \mu(w,c)$.
\end{lemma}

\begin{proof}
    Lemma~\ref{lem:saxpstruct} gives us the existence of an SMS $\X$ for $w \in \borda(G)$ such that $\X \subseteq \edge{E}{w}{}$.
    Suppose $\exists c_0 \in \cstar,\, \delta(c_0,\X) \leq 0$.
    Since $\X$ is an MS, $\forall c \in \cstar,\, \delta(c,\X) \geq 0$. Hence, $\delta(c_0,\X) = 0$, i.e., $\sum_{c\in \cstar} \mu_{\X}(w,c) = \n(\m - 1) - \mu_{\X}(w,c_0)$. 
    Additionally, $\forall c \in \cstar,\, \delta(c,\X) \geq 0 = \delta(c_0,\X)$ and we have $\forall c \in \cstar,\, \mu_{\X}(w,c) \geq \mu_{\X}(w,c_0)$.
    Hence, we have $\forall c \in \cstar,\, |\X| \leq \n(\m-1) - \mu_{\X}(w,c)$ and in particular, $|\X| \leq \n(\m-1) - \min_{c\in \cstar} \mu_{\X}(w,c)$.

    Note that in the last formula we have $\min_{c\in \cstar} \mu_{\X}(w,c)$ and not $\min_{c\in \cstar} \mu(w,c)$. We still have to show that there is an SMS such that there exists a candidate $c$ distinct from $w$ such that $\mu_{\X}(w,c)=\mu(w,c)$.
    
    Suppose for all SMSs $\Y \subseteq \edge{E}{w}{}$, $|\Y| > \n(\m-1) - \min(\lfloor n(1-1/m)\rfloor, \min_{c\neq w}\mu(w,c))$. 
    If $\forall c\in\candidates\setminus{\{w\}}$, $\sm{c}{\Y} \leq 1$, we are done thanks to Lemma~\ref{lem:case2}. Otherwise, we have the existence of $c_1 \neq w$ such that $\delta(c_1,\Y)>1$ and by minimality of MSs, $\exists c_0 \neq w$ such that $\delta(c_0,\Y) = 0$.
    Suppose now  that $\forall c \neq w$ such that $\delta(c,w) = 0$, $\mu_{\Y}(w,c)<\mu(w,c)$.
    Consider $\Y'$ such that $\mu_{\Y'}(w,c_0) = \mu_{\Y}(w,c_0) + 1$, $\mu_{\Y'}(w,c_1) = \mu_{\Y}(w,c_1) - 1$ (which is possible as $\mu_{\Y'}(w,c_1) < \mu(w,c_1)$) and else $\mu_{\Y'} = \mu_{\Y}$.
    $\delta(c_0,\Y') = 1$, $\delta(c_1,\Y') = \delta(c_1,\Y) - 1 \geq 1$ and $\forall c\not\in\{w,c_0,c_1\},\, \delta(c,\Y') = \delta(c_1,\Y) \geq 1$. Since all deltas are positive, $\Y'$ is a weak minimal support and since $|\Y'| = |\Y|$ and $\Y$ is an SMS, $\Y'$ is an SMS too.
    By passing from $\Y$ to $\Y'$ we strictly reduce the number of losing candidates with a null delta ($\delta(c_0,\Y) = 0$ but $\delta(c_0,\Y') = 1$).
    Since the number of candidates is finite, this process has to stop at one point, i.e., $\exists \Y^*$ an SMS such that $\exists c^* \neq w$ with $\delta(c^*, \Y^*) = 0$ and $\mu_{\Y}(w,c^*)=\mu(w,c^*)$. 
    We have $\sum_{c\neq w} \mu_{\Y^*}(w,c^*) - \n(\m-1) + \mu(w,c^*) = 0$ and $|\Y^*| = \n(\m-1) - \mu(w,c^*)$. 
    Finally, since $\min_{c\neq w} \mu(w,c) \leq \mu(w,c^*)$, exists an SMS $\X$ such that, $|\X| \leq \n(\m-1) - \min_{c\neq w} \mu(w,c)$.  
\end{proof}

We are finally ready to show our theorem.

\bordasize*

\begin{proof}
    Let $\X$ be an SMS for $w \in \borda(G)$.
    With Lemma~\ref{lem:case1}, we show that if $\sigma_w \leq \n(\m-1)-\min_{c\neq w}\mu(w,c)$ then the lower bound of Lemma~\ref{lem:lowerbound1} is reached and for all SMSs $\X$ for $w \in \borda(G)$,  $|\X|{=}(\m-1)(\n(\m-1)-\sigma_w)$. 
    
    With Lemma~\ref{lem:subcases}, we show that the case where $\sigma_w > \n(\m-1)-\min_{c\neq w}\mu(w,c)$ can be further divided into two sub-cases. 
    With Lemma~\ref{lem:case2} and Lemma~\ref{lem:lowerbound2} or Lemma~\ref{lem:case3} and Lemma~\ref{lem:lowerbound3}, we show that either $|\X| = \left\lceil \frac{\n(\m-1)^2}{\m} \right\rceil = \left\lceil \n(\m-1) -  \n (\m-1)/\m \right\rceil = \n(\m-1) - \left\lfloor \n (\m-1)/\m \right\rfloor = \n(\m-1) - \lfloor \n(1-1/\m)\rfloor$ or $|\X| = \n(\m-1) - \min_{c\neq w} \mu(w,c)$.
    In the theorem, we chose to merges those two sub-cases to obtain a more general case where $|\X|{=} n(m-1) - \min(\lfloor n(1-1/m)\rfloor, \min_{c\neq w}\mu(w,c))$.
\end{proof}

\bordabounds*

\begin{proof}
    The lower bound directly result from the Theorem~\ref{th:bordasize} as $\n(\m-1) - \lfloor \n(1-1/\m)\rfloor = \n(\m-1) - \left\lfloor \n (\m-1)/\m \right\rfloor = \left\lceil \n(\m-1) -  \n (\m-1)/\m \right\rceil = \left\lceil \frac{\n(\m-1)^2}{\m} \right\rceil$.
    For the upper bound, note that the Borda score of the winner $\sigma_w$ is greater than the average score of every candidate, i.e., $\sigma_w \geq \frac{1}{\m} \n \frac{\m(\m-1)}{2}$. In the worst case $\sigma_w=\left\lceil\n\frac{\m-1}{2}\right\rceil$ and $|\X| = (\m-1)(\n(\m-1)-\left\lceil\n\frac{\m-1}{2}\right\rceil) = (\m-1)\left\lfloor\n\frac{\m-1}{2}\right\rfloor$.
\end{proof}

\bordaalgo*

\begin{proof}
    To compute a maxwin-SMS in $\mathcal{O}(\m^2\log \n)$, we propose the following Algorithm~\ref{alg:findSAXpBorda} which can naturally be derived from the analysis carried in the proof of Theorem~\ref{th:bordasize} on the size of SMSs for the Borda rule.
\end{proof}

\begin{algorithm}
\caption{\textsc{findMaxwinSMS-Borda} 
}\label{alg:findSAXpBorda}
\KwData{Complete $n$-weighted tournament $G =(\candidates,\mu)$, winning candidate $w \in \borda(G)$}
\KwResult{a maxwin-SMS $\X=(\candidates,\mu_{\X})$}
$\mu_{\X}(\cdot,\cdot) \gets 0$\;
\For{$c \in \cstar$}{
        $\mu_{\X}(w,c) \gets \mu(w,c)$\;
}
\eIf{$\exists c \in \cstar,\, \sm{c}{\X} < 0$}{
    \For{$c \in \cstar$}{
        \While{$\sm{c}{\X} < 0$}{
            $\mu_{\X}(c',c) \gets \min\big(\mu_{\X}(c',c) - \sm{c}{\X}, \mu(c',c)\big)$ 
            $\text{ for some } c' \text{ s.t. } \mu_{\X}(c',c)<\mu(c',c)$\;
        }
    }
}{
    $\text{actual} \gets \sum_{c\in\candidates} \mu_\X(w,c)$\;
    $\text{final} \gets n(m-1) - \min(\lfloor n(1-\frac{1}{m})\rfloor, \displaystyle\min_{c\neq w}\mu(w,c))$\;
    $\text{rm} \gets \text{actual} - \text{final}$\;
    $\text{rmed} \gets 0$\;
    \While{$\text{rmed}<\text{rm}$}{
        $\text{iteRm} \gets \min\big(\text{rm} - \text{rmed},\sm{c}{\X} - \text{rm}\big)$ 
        $\text{ for some } c \text{ s.t. } \sm{c}{\X} > \text{rm}$\;
        $\mu_{\X}(w,c) \gets \mu_{\X}(w,c) - \text{iteRm}$\;
        $\text{rmed} \gets \text{rmed} + \text{iteRm}$\;
    }
}
\Return{$\X$}
\end{algorithm}

To prove Proposition~\ref{prop:wcstructure}, we introduce a lemma to characterize how SMSs and $\edge{E}{w}{}$ interact in different cases.

\begin{restatable}{lemma}{Bordacharacterisation}\label{lemma:Bordacharacterisation}
    Given $n$ voters, a complete $n$-weighted tournament $G =(\candidates,E)$ , and a winning candidate $w \in \borda(G)$, for all SMSs $\X$ for $w \in \borda(G)$,
    \begin{enumerate}[label=(\roman*)]
        \item if $\exists c_0,c_1 \in \cstar$, $c_0\neq c_1$ s.t. $\sm{c_0}{\edge{E}{w}{}} < 0$ and $\sm{c_1}{\edge{E}{w}{}} < 0$ then $\edge{E}{w}{} \subseteq \X$
        \item else if $\exists! c_0\in \cstar$ s.t. $\sm{c_0}{\edge{E}{w}{}} \leq 0$ then 
            $|\edge{E}{w}{}\setminus{\X}|\leq \min_{c \in \candidates\setminus{\{w,c_0\}}} \sm{c}{\edge{E}{w}{}}$ and exists an SMS containing $\edge{E}{w}{}$
        \item else for all SMSs $\Y$ for $w \in \borda(G)$, $|\WC(\X)-\WC(\Y)|\leq 1$.
    \end{enumerate}
\end{restatable}

\begin{proof}
    \textit{(i)} We proceed by contradiction. Suppose, there is an SMS $\X$ of $w \in \borda(G)$ such that $\edge{E}{w}{} \not \subseteq \X$. Then, $\exists c_2 \in \cstar$ such that $(w,c_2) \in E$ and $(w,c_2) \not \in \X$.
    Additionally, we know that $\exists e_0 \in \edge{\X}{}{c_0}$ such that $e_0 \neq (w,c_0)$ (all $c_0$'s defeats in $\X$ cannot come from $w$ since $\sm{c_0}{\edge{E}{w}{}} < 0$). Similarly, $\exists e_1 \in \edge{\X}{}{c_1}$, $e_1 \neq (w,c_1)$.
    Let $\X' = \X \setminus{\{e_0,e_1\}}\cup\{(w,c_2)\}$. $\forall c\in \candidates\setminus{\{w,c_0,c_1,c_2\}},\, \sm{c}{\X'} = \sm{c}{\X} + 1 \geq 0$. $\sm{c_2}{\X'} = \sm{c_2}{\X} + 1 \geq 0$. $\forall c\in \{c_0,c_1\}\setminus{\{c_2\}}$, $\sm{c}{\X'} = \sm{c}{\X} - 1 + 1 \geq 0$. Hence $\X'$ is a weak MS and it contains at least one MS with a smaller cardinality. However, $|\X'|<|\X|$ and $\X$ is an SMS. Contradiction.
    
    \textit{(ii)} If there exists a unique $c_0\in \cstar$ s.t. $\sm{c_0}{\edge{E}{w}{}} \leq 0$, taking $\edge{E}{w}{}$ and adding just enough element of $\edge{E}{}{c_0}$ so that $\sm{c_0}{\X} = 0$ gives an SMS (correctness of Algorithm~\ref{alg:findSAXpBorda}). We just have to measure by how much we can "deviate" from this SMS by swapping out elements of $\edge{E}{w}{}$ for elements of $\edge{E}{}{c_0}$ to keep $\sm{c_0}{\X} = 0$. Basically, by each swap out we reduce every score margin by at least 1. So at best, it is doable while the second smallest score margin remains positive i.e. $\min_{c \in \candidates\setminus{\{w,c_0\}}} \sm{c}{\edge{E}{w}{}}$ times. Note that this bound is not necessarily tight because we have to pick a specific candidate $c_1 \neq c_0$ whose score margin would decrease by 2 instead of 1 when $(w,c_1)$ gets swapped out.
    

    \textit{(iii)} Here, $\forall c\neq w,\, \sm{c}{\edge{E}{w}{}} \geq 0$ and Algorithm~\ref{alg:findSAXpBorda} returns an SMS $\X \subseteq \edge{E}{w}{}$ s.t. $\max_{c \in \cstar} \sm{c_{0}}{\X} \leq 1 \lor \min_{c \in \cstar} \sm{c}{\X} \leq 0$.
    If $\min_{c \in \cstar} \sm{c}{\X} \leq 0$, we work our way with arguments similar to those introduced in \textit{(ii)}.
    If $\max_{c \in \cstar} \sm{c_{0}}{\X} \leq 1$, $\sum_{c \in \cstar}\sm{c_{0}}{\X} \leq m-1$. We recall that each element of $\edge{E}{w}{}$ contributes $m$ to $\sum_{c\neq w} \sm{c}{\X}$. Hence, only one swap out of an element of $\edge{E}{w}{}$ is possible before having a negative score margin for one candidate (decreasing the sum by $m$ and increasing it by $1$ with the swap in of an element not in $\edge{E}{w}{}$).
\end{proof}

\wcstructure*

\begin{proof}
    As we have shown that any SMS returned by Algorithm~\ref{alg:findSAXpBorda} is optimal regarding win count, it suffices to show that the maximum difference in terms of wins between the outcome of Algorithm~\ref{alg:findSAXpBorda} and any other SMS is smaller than $\frac{n}{2}$.
    The proof of this proposition uses the characterisation of SMSs for the Borda rule shown in Lemma~\ref{lemma:Bordacharacterisation}. In case \textit{(i)}, the maximum difference is $0$. In case \textit{(iii)}, it is $1$. In case \textit{(ii)}, it is at most $\min_{c \in \candidates\setminus{\{w,c_0\}}} \sm{c}{\edge{E}{w}{}}$. Let us show that this is less than $\frac{\n}{2}$. Suppose that $\exists! c_0\in \cstar$ s.t. $\sm{c_0}{\edge{E}{w}{}} \leq 0$, let $\X$ be the SMS returned by Algorithm~\ref{alg:findSAXpBorda}. We have that $\edge{E}{w}{} \subseteq \X$. Thus, by minimality of an MS we have $\forall c\in\cstar$,$\forall c'\in\candidates\setminus{\{w,c_0\}}$, $\mu(c,c')=0$. Hence, $\forall c \in \candidates\setminus{\{w,c_0\}}$, $\sm{c}{\X}=\sum_i\mu(w,i) + \sum_i \mu(i,c) - \n(\m-1) = \sum_i\mu(w,i) + \mu(w,c) - \n(\m-1)$ and $\sm{c_0}{\edge{E}{w}{}} \leq 0$ can be rewritten as $\sum_{i\neq c_0}\mu(w,i) + 2\mu(w,c_0) \leq \n(\m-1)$. Additionally, $\sum_{i\neq c_0}\mu(w,i) \leq \n(\m-2)$. 
    Hence, 
    \begin{align*}
        & \sum_{i\neq c_0}\mu(w,i) + 2\mu(w,c_0) + \sum_{i\neq c_0}\mu(w,i) \\
        & \leq \n(\m-1) + \n(\m-2)\\
        \implies & 2\sum_{i}\mu(w,i)\leq 2\n(\m-1)-\n\\
        \implies & \sum_{i}\mu(w,i)\leq \n(\m-1)-\frac{\n}{2}\\
        \implies & \sum_{i}\mu(w,i)-\n(\m-1)\leq -\frac{\n}{2}
    \end{align*}
    Finally, $\forall c \in \candidates\setminus{\{w,c_0\}}$, $\sum_{i}\mu(w,i)-(\n(\m-1)-\mu(w,c))\leq \mu(w,c) -\frac{\n}{2} \leq \n - \frac{\n}{2} = \frac{\n}{2}$.
\end{proof}

\copsize*

\begin{proof}
    When $w$ is a Condorcet winner, we are in the second case of Theorem~\ref{th:bordasize}. Indeed, in that case, $\sigma_w = \m-1$, $\min_{c\neq w} \mu(w,c) = 1$ and we have $\m-1> 1(\m-1) - 1$.
    Additionally, $\left\lfloor1(1-1/\m)\right\rfloor = 0$ and $|\X|=\m-1$.
    The other case is clear.
\end{proof}

\copalgo*

\begin{proof}
    Algorithm~\ref{alg:findSAXpCop} is a simplified version of Algorithm~\ref{alg:findSAXpBorda}, where we only consider the first branch of the if on line 4. We only reach the second branch when $w$ is a Condorcet winner and in that case $\{(w,c):c\in\candidates\}$ is an SMS.
\end{proof}

\begin{algorithm}
\caption{\textsc{findMaxwinSMS-Cop} 
}\label{alg:findSAXpCop}
\KwData{Complete tournament $G =(\candidates,E)$, winning candidate $w \in \cop(G)$}
\KwResult{a maxwin-SMS $\X$}
$\X \gets \edge{E}{w}{}$\;
\For{$c \in \candidates\setminus{\{w\}}$}{
    \While{$|\edge{\X}{}{c}| <  m-1-\sigma_{\cop}(w)$}{
    $\X \gets \X \cup \{(c',c)\} \text{ for some } (c',c) \in E\setminus{\X}$\;
    }
}
\Return{$\X$}
\end{algorithm}

\section{Additional Examples for Section~\ref{sec:explanation}}\label{sec:app_expl}

\begin{table*}
    \centering
    \begin{tabular}{c|c|c|c|c}
        Tournament & Rule & SMS & Structure & Natural Language Explanation\\ \hline
        \multirow{3}{*}[-2.3cm]{\begin{tikzpicture}[align=center]
    \node [node] (a) at (0,0) {$a$};
    \node [node] (b) at (2,0) {$b$};
    \node [node] (c)  at (2,-2) {$c$};
    \node [node] (d) at (0,-2) {$d$};

    \draw [arrow] (a) to (b) ;

    \draw [arrow] (c) to (a) ;
    
    \draw [arrow] (a) to (d) ;
    
    \draw [arrow] (b) to (c) ;

    \draw [arrow] (b) to (d) ;
    
    \draw [arrow] (c) to (d) ;
\end{tikzpicture}} 
        & $\tc$ & \raisebox{-0.5\height}{\begin{tikzpicture}[align=center]
    \node [node] (a) at (0,0) {$a$};
    \node [node] (b) at (2,0) {$b$};
    \node [node] (c)  at (2,-2) {$c$};
    \node [node] (d) at (0,-2) {$d$};

    \draw [arrow] (a) to (b) ;

    
    \draw [arrow] (a) to (d) ;
    
    \draw [arrow] (b) to (c) ;

    
\end{tikzpicture}} & \raisebox{-0.5\height}{\begin{tikzpicture}[align=center, sibling distance=15mm, level distance = 10mm]
    \node [node] (a) at (0,0) {$a$}
        [arrow] child {node [node] (b) {$b$} 
            [arrow] child {node [node] (c) {$c$}}}
        child {node [node] (d) {$d$}};
\end{tikzpicture}} &
        \begin{minipage}[]{0.45\linewidth}
                $a$ is part of the top cycle because eliminating $a$ would lead to an empty top cycle. Indeed eliminating $a$ would
            \begin{itemize}
                \item eliminate $b$ because $a$ is preferred to $b$ ($(a,b)\in\X$)
                \begin{itemize}
                    \item which would eliminate $c$ because $b$ is preferred to $c$ ($(b,c)\in\X$)
                \end{itemize}
                \item eliminate $d$ because $a$ is preferred to $d$ ($(a,d)\in\X$)
            \end{itemize}
        \end{minipage}
        \\ \cline{2-5}
        & $\uc$ & \raisebox{-0.5\height}{\begin{tikzpicture}[align=center]
    \node [node] (a) at (0,0) {$a$};
    \node [node] (b) at (2,0) {$b$};
    \node [node] (c)  at (2,-2) {$c$};
    \node [node] (d) at (0,-2) {$d$};

    \draw [arrow] (a) to (b) ;

    
    \draw [arrow] (a) to (d) ;
    
    \draw [arrow] (b) to (c) ;

    
\end{tikzpicture}} & \raisebox{-0.5\height}{\begin{tikzpicture}[align=center, sibling distance=15mm, level distance = 10mm]
    \node [node] (a) at (0,0) {$a$}
        [arrow] child {node [node] (b) {$b$} 
            [arrow] child {node [node] (c) {$c$}}}
        child {node [node] (d) {$d$}};
\end{tikzpicture}} &
        \begin{minipage}[]{0.45\linewidth}
            $a$ is part of the uncovered set because
            \begin{itemize}
                \item $a$ is not covered by $b$ since $a$ is preferred to $b$ ($(a,b)\in\X$)
                \item $a$ is not covered by $c$ since $a$ is preferred to $b$ ($(a,b)\in\X$) and $b$ is preferred to $c$ ($(b,c)\in\X$)
                \item $a$ is not covered by $d$ since $a$ is preferred to $d$ ($(a,d)\in\X$)
            \end{itemize}
        \end{minipage}
        \\ \cline{2-5}
        & $\cop$ & \raisebox{-0.5\height}{\begin{tikzpicture}[align=center]
    \node [node] (a) at (0,0) {$a$};
    \node [node] (b) at (2,0) {$b$};
    \node [node] (c)  at (2,-2) {$c$};
    \node [node] (d) at (0,-2) {$d$};

    \draw [arrow] (a) to (b) ;

    
    \draw [arrow] (a) to (d) ;
    
    \draw [arrow] (b) to (c) ;

    
\end{tikzpicture}} & \raisebox{-0.5\height}{\begin{tikzpicture}[align=center, sibling distance=5mm, level distance = 8mm]
    \node [node] (a) at (0,0) {$a$}
        [arrow] child {coordinate (a1) {}}
        child {coordinate (a2)};

    \node [node] (b) at (1.5,0) {$b$}
        [arrow, {Stealth[length=5pt,round,inset=3pt,width=7pt,flex'=1]}-] child {coordinate (b1) {}};

    \node [node] (d) at (0,-1.5) {$b$}
        [arrow, {Stealth[length=5pt,round,inset=3pt,width=7pt,flex'=1]}-] child {coordinate (b1) {}};

    \node [node] (c) at (1.5,-1.5) {$b$}
        [arrow, {Stealth[length=5pt,round,inset=3pt,width=7pt,flex'=1]}-] child {coordinate (b1) {}};
\end{tikzpicture}} &
        \begin{minipage}[]{0.45\linewidth}
            $a$ is part of the Copeland winners because
            \begin{itemize}
                \item $a$ wins at least 2 head-to-heads since $a$ is preferred to $b$ and to $d$ ($(a,b)\in\X$, $(a,d)\in\X$)
                \item $b$ wins at most 2=3-1 head-to-heads since it loses at least 1 head-to-head to $a$ ($(a,b)\in\X$)
                \item $c$ wins at most 2=3-1 head-to-heads since it loses at least 1 head-to-head to $b$ ($(b,c)\in\X$)
                \item $d$ wins at most 2=3-1 head-to-heads since it loses at least 1 head-to-head to $a$ ($(a,d)\in\X$)
            \end{itemize}
        \end{minipage}
        \\ \hline
        \multirow{3}{*}[-2.9cm]{\begin{tikzpicture}[align=center]
    \node [node] (a) at (0,0) {$a$};
    \node [node] (b) at (2,0) {$b$};
    \node [node] (c)  at (2,-2) {$c$};
    \node [node] (d) at (0,-2) {$d$};

    \draw [arrow] (a) to [out=15, in=165] node[quotes]  {$3$} (b) ;
    \draw [arrow] (b) to [out=-165, in=-15] node[quotes]  {$2$} (a) ;

    \draw [arrow] (a) [out=-30, in=120] to node[quotes,pos=0.2] {$2$} (c) ;
    \draw [arrow] (c) [out=150, in=-60] to node[quotes,pos=0.2] {$3$} (a) ;
    
    \draw [arrow] (a) [out=-75, in=75] to node[quotes] {$4$} (d) ;
    \draw [arrow] (d) [out=105, in=-105] to node[quotes] {$1$} (a) ;
    
    \draw [arrow] (b) [out=-75, in=75] to node[quotes] {$3$} (c) ;
    \draw [arrow] (c) [out=105, in=-105] to node[quotes] {$2$} (b) ;

    \draw [arrow] (b) [out=-120, in=30] to node[quotes,pos=0.2] {$3$} (d) ;
    \draw [arrow] (d) [out=60, in=-150] to node[quotes,pos=0.2] {$2$} (b) ;
    
    \draw [arrow] (d) to [out=15, in=165] node[quotes]  {$3$} (c) ;
    \draw [arrow] (c) to [out=-165, in=-15] node[quotes]  {$2$} (d) ;
\end{tikzpicture}}
        & $\mm$ & \raisebox{-0.5\height}{\begin{tikzpicture}[align=center]
    \node [node] (a) at (0,0) {$a$};
    \node [node] (b) at (2,0) {$b$};
    \node [node] (c)  at (2,-2) {$c$};
    \node [node] (d) at (0,-2) {$d$};

    \draw [arrow] (a) to node[quotes]  {$3$} (b) ;

    \draw [arrow] (a) to node[quotes] {$2$} (c) ;
    
    \draw [arrow] (a) to node[quotes] {$3$} (d) ;
    
    \draw [arrow] (b) to node[quotes] {$3$} (c) ;

    
\end{tikzpicture}} & \raisebox{-0.5\height}{\begin{tikzpicture}[align=center, sibling distance=5mm, level distance = 10mm]
    \node [node] (a) at (0,0) {$a$}
        [arrow] child {coordinate (a1) edge from parent node[quotes,pos=0.4] {$3$}}
        child {coordinate (a2) edge from parent node[quotes,pos=0.4] {$2$}}
        child {coordinate (a3) edge from parent node[quotes,pos=0.4] {$3$}};

    \node [node] (b) at (1.5,0) {$b$}
        [arrow, {Stealth[length=5pt,round,inset=3pt,width=7pt,flex'=1]}-] child {coordinate (b1) edge from parent node[quotes,pos=0.6] {$3$}};

    \node [node] (d) at (0,-1.5) {$d$}
        [arrow, {Stealth[length=5pt,round,inset=3pt,width=7pt,flex'=1]}-] child {coordinate (b1) edge from parent node[quotes,pos=0.6] {$3$}};

    \node [node] (c) at (1.5,-1.5) {$c$}
        [arrow, {Stealth[length=5pt,round,inset=3pt,width=7pt,flex'=1]}-] child {coordinate (b1) edge from parent node[quotes,pos=0.6] {$3$}};
\end{tikzpicture}} &
        \begin{minipage}[]{0.45\linewidth}
            $a$ is part of the maximin set because
            \begin{itemize}
                \item $a$ wins at least 2 pairwise comparisons in each head-to-head ($\mu_\X(a,b)=3$, $\mu_\X(a,c)=2$, $\mu_\X(a,d)=3$)
                \item $b$ wins at most 2=5-3 pairwise comparisons against $a$ ($\mu_\X(a,b)=3$)
                \item $c$ wins at most 2=5-3 pairwise comparisons against $b$ ($\mu_\X(b,c)=3$)
                \item $d$ wins at most 2=5-3 pairwise comparisons against $a$ ($\mu_\X(a,d)=3$)
            \end{itemize}
        \end{minipage}
        \\ \cline{2-5}
        & $\wuc$ & \raisebox{-0.5\height}{\begin{tikzpicture}[align=center]
    \node [node] (a) at (0,0) {$a$};
    \node [node] (b) at (2,0) {$b$};
    \node [node] (c)  at (2,-2) {$c$};
    \node [node] (d) at (0,-2) {$d$};

    \draw [arrow] (a) to node[quotes]  {$3$} (b) ;

    
    \draw [arrow] (a) to node[quotes] {$3$} (d) ;
    
    \draw [arrow] (b) to node[quotes] {$3$} (c) ;

    
\end{tikzpicture}} & \raisebox{-0.5\height}{

\begin{tikzpicture}[align=center, sibling distance=15mm, arrow, pos=0.4, level 1/.style={level distance=13mm}, level 2/.style={level distance=13mm}]
    \node [node] (a) at (0,0) {$a$}
        child {node [node] (b) {$b$} 
            child {node [node] (c) {$c$} edge from parent node[quotes] {$3$}}
            edge from parent node[quotes] {$3$}}
        child {node [node] (d) {$d$} edge from parent node[quotes] {$3$}};
\end{tikzpicture}} &
        \begin{minipage}[]{0.45\linewidth}
            $a$ is part of the weighted uncovered set because
            \begin{itemize}
                \item $a$ is not weighted covered by $b$ because $a$ is more strongly preferred over $d$ than $b$ ($\mu_\X(a,d)=4$, $\mu_\X(d,b)=2$)
                \item $a$ is not weighted covered by $c$ because $a$ is more strongly preferred over $d$ than $c$ ($\mu_\X(a,d)=4$, $\mu_\X(d,c)=2$)
                \item $a$ is not weighted covered by $d$ because $a$ is preferred in strict majority over $d$ ($\mu_\X(a,d)=4$)
            \end{itemize}
        \end{minipage}
        \\ \cline{2-5}
        & $\borda$ & \raisebox{-0.5\height}{\begin{tikzpicture}[align=center]
    \node [node] (a) at (0,0) {$a$};
    \node [node] (b) at (2,0) {$b$};
    \node [node] (c)  at (2,-2) {$c$};
    \node [node] (d) at (0,-2) {$d$};

    \draw [arrow] (a) to node[quotes]  {$3$} (b) ;

    \draw [arrow] (a) to node[quotes,pos=0.2] {$2$} (c) ;
    
    \draw [arrow] (a) to node[quotes] {$4$} (d) ;
    
    \draw [arrow] (b) [out=-75, in=75] to node[quotes] {$3$} (c) ;
    \draw [arrow] (c) [out=105, in=-105] to node[quotes] {$2$} (b) ;

    \draw [arrow] (b) [out=-120, in=30] to node[quotes,pos=0.8] {$2$} (d) ;
    \draw [arrow] (d) [out=60, in=-150] to node[quotes,pos=0.8] {$1$} (b) ;
    
    \draw [arrow] (d) to node[quotes]  {$1$} (c) ;
\end{tikzpicture}} & \raisebox{-0.5\height}{\begin{tikzpicture}[align=center, sibling distance=5mm, level distance = 10mm]
    \node [node] (a) at (0,0) {$a$}
        [arrow] child {coordinate (a1) edge from parent node[quotes,pos=0.4] {$3$}}
        child {coordinate (a2) edge from parent node[quotes,pos=0.4] {$2$}}
        child {coordinate (a3) edge from parent node[quotes,pos=0.4] {$4$}};
        
    \node [node] (b) at (1.5,0) {$b$}
        [arrow, {Stealth[length=5pt,round,inset=3pt,width=7pt,flex'=1]}-] child {coordinate (b1) edge from parent node[quotes,pos=0.6] {$3$}}
        child {coordinate (b2) edge from parent node[quotes,pos=0.6] {$2$}}
        child {coordinate (b3) edge from parent node[quotes,pos=0.6] {$1$}};
        
    \node [node] (c) at (1.5,-1.5) {$c$}
        [arrow, {Stealth[length=5pt,round,inset=3pt,width=7pt,flex'=1]}-] child {coordinate (c1) edge from parent node[quotes,pos=0.6] {$2$}}
        child {coordinate (c2) edge from parent node[quotes,pos=0.6] {$3$}}
        child {coordinate (c3) edge from parent node[quotes,pos=0.6] {$1$}};
        
    \node [node] (d) at (0,-1.5) {$d$}
        [arrow, {Stealth[length=5pt,round,inset=3pt,width=7pt,flex'=1]}-] child {coordinate (d1) edge from parent node[quotes,pos=0.6] {$4$}}
        child {coordinate (d2) edge from parent node[quotes,pos=0.6] {$2$}};
\end{tikzpicture}} &
        \begin{minipage}[]{0.45\linewidth}
            $a$ is part of the Borda winners because
            \begin{itemize}
                \item $a$ wins at least 9=3+2+4 pairwise comparisons ($\mu_\X(a,b)=3$, $\mu_\X(a,c)=2$, $\mu_\X(a,d)=4$)
                \item $b$ wins at most 9=5*3-6 pairwise comparisons since it loses at least 6=3+2+1 ($\mu_\X(a,b)=3$, $\mu_\X(c,b)=2$, $\mu_\X(d,b)=1$)
                \item $c$ wins at most 9=5*3-6 pairwise comparisons since it loses at least 6=2+3+1 ($\mu_\X(a,c)=2$, $\mu_\X(b,c)=3$, $\mu_\X(d,c)=1$)
                \item $d$ wins at most 9=5*3-6 pairwise comparisons since it loses at least 6=4+2 ($\mu_\X(a,d)=4$, $\mu_\X(b,d)=2$)
            \end{itemize}
        \end{minipage}
    \end{tabular}
    \caption{Examples of smallest minimal supports for various tournament solutions, their underlying structure and an associated natural language explanation.}
    \label{tab:mss}
\end{table*}

\end{document}